\documentclass[conference]{IEEEtran}
\IEEEoverridecommandlockouts
\usepackage{cite}
\usepackage{amsmath,amssymb,amsfonts}

\newcommand{\wrap}{\operatorname{wrap}}
\usepackage{amsthm}

\theoremstyle{plain}
\newtheorem{theorem}{Theorem}[section]

\newtheorem{lemma}[theorem]{Lemma}

\theoremstyle{definition}

\newtheorem{assumption}[theorem]{Assumption}
\theoremstyle{remark}

\usepackage{cite}
\usepackage{amsmath,amssymb,amsfonts}
\usepackage{graphicx}
\usepackage{subcaption}
\usepackage{textcomp}
\usepackage{xcolor}
\usepackage{booktabs}
\usepackage{algorithm}
\usepackage{algpseudocode}

\usepackage{multirow}

\usepackage{algcompatible}
\usepackage{subfig}
\def\BibTeX{{\rm B\kern-.05em{\sc i\kern-.025em b}\kern-.08em
    T\kern-.1667em\lower.7ex\hbox{E}\kern-.125emX}}
\begin{document}

\title{A Stable Aggregation Method for Quantum Federated Learning\\}
\author{\IEEEauthorblockN{Shanika Nanayakkara and Shiva Raj Pokhrel}\thanks{Authors are with the IoT \& Software Engineering Research Lab, Deakin University, Geelong, VIC, Australia; email: s.nanayakkara@deakin.edu.au, shiva.pokhrel@deakin.edu.au}}

\maketitle

% \begin{abstract}
% Federated learning (FL) over quantum-enabled and heterogeneous communication
% networks is challenged by uneven client quality, stochastic fidelity variation,
% device instability, and the periodic geometry of variational quantum model
% parameters. Classical aggregation rules typically assume Euclidean parameter
% spaces and uniform communication reliability, which can lead to geometrically
% inconsistent updates when quantum rotation parameters lie near angular wrap-around
% boundaries. This paper proposes a self-consistent midpoint adaptive aggregation
% framework for quantum federated learning. The method integrates
% quality-of-service-aware client weighting, torus-consistent angular aggregation,
% and bounded midpoint refinement to determine a stable next server state rather
% than directly jumping to the client aggregation target. Controlled angular stress
% tests show that the proposed method preserves benign-case behaviour, corrects
% seam-induced Euclidean failures, and achieves small fixed-point residuals.
% Experiments on medical and financial datasets demonstrate improved
% accuracy--stability behaviour, while IBM quantum hardware validation confirms
% that the controlled angular outputs remain physically meaningful under
% real-device execution. The results indicate that self-consistent,
% QoS-aware, and geometry-aware aggregation can improve the stability of
% federated learning in emerging quantum and heterogeneous network environments.
% \end{abstract}
\begin{abstract}
Quantum federated learning (QFL) enables clients to train quantum neural network (QNN) models without sharing private data. 
We find that aggregation in QFL is unstable under heterogeneous data, unreliable communication, variable fidelity, latency, and quantum hardware noise. Moreover, QFL is non-trivially challenging because several QNN parameters are periodic angles, where Euclidean averaging often fails to capture the inherent dynamics. We develop a novel self-consistent midpoint aggregation method for stable QFL design and implementation. 
We combine QoS-aware client weighting, circular parameter aggregation, and bounded midpoint-based update control. We perform several angular tests and IBM real Quantum machines experiments for validation confirming our approach. Extensive evaluations and experiments on medical and financial datasets show improved stability, lower volatility, and competitive accuracy. 

\end{abstract}
\begin{IEEEkeywords}
Quantum federated learning, adaptive aggregation, geometry-aware aggregation,
self-consistent midpoint, quality of service, angular aggregation, quantum neural
networks.
\end{IEEEkeywords}

\section{Introduction}

Federated learning (FL)~\cite{b1,b3, 10758814, 11002682} enables multiple clients to train a shared global model without exposing their private data. 
Its practical performance, however, is strongly limited by the quality of server-side aggregation. 
Even in classical FL, with Federated Averaging (FedAvg~\cite{b1}), the aggregation is affected by non-IID data, unbalanced client samples, partial participation, and heterogeneous computing or communication resources~\cite{11029580,10648926}. 
These factors create biased local updates, client drift, and unstable convergence. 
This has motivated robust aggregation and correction methods such as FedProx~\cite{li2020fedprox} and SCAFFOLD~\cite{b3}.

In quantum federated learning (QFL)~\cite{11002682, chehimi2022qfl, 10758814}, the aggregation problem illustrated in Fig~\ref{fig:figg2} becomes more challenging. Observe in Fig~\ref{fig:figg2} that
QFL extends federated optimization to distributed quantum and hybrid quantum--classical models, where clients train parameterized quantum circuits and the server aggregates quantum model parameters~\cite{chehimi2022qfl,chehimi2023foundations,quantumfed2021, 10758814}. 
Unlike classical model parameters, many quantum neural network (QNN) parameters are rotation angles with periodic geometry. 
Thus, direct Euclidean averaging can produce geometrically inconsistent updates, especially near angular wrap-around boundaries. 
Moreover, client updates in quantum-enabled networks are affected not only by data heterogeneity but also by quantum-channel fidelity, latency, decoherence, shot noise, and device instability~\cite{11314201}. 
Therefore, QFL aggregation must be both reliability-aware and geometry-aware. 
This motivates the need for aggregation rules that account for client quality while respecting the circular or torus-valued structure of quantum model parameters.

In such QFL settings, the quality of a
client update may depend not only on data heterogeneity but also on quantum
communication and device conditions, including teleportation fidelity~\cite{11480649},
entanglement quality, latency, decoherence, and noisy intermediate-scale quantum
(NISQ) device instability~\cite{bennett1993teleportation,kimble2008quantum,
wehner2018quantum_internet,preskill2018nisq}. Consequently, treating all client
updates as equally reliable, or aggregating them only according to local data
size, can be suboptimal. A QFL aggregation rule should therefore account for
both learning-side heterogeneity and physical-layer reliability when determining
the contribution of each client to the global quantum model.
\begin{figure}[t]
    \centering
    \includegraphics[width=0.890\linewidth]{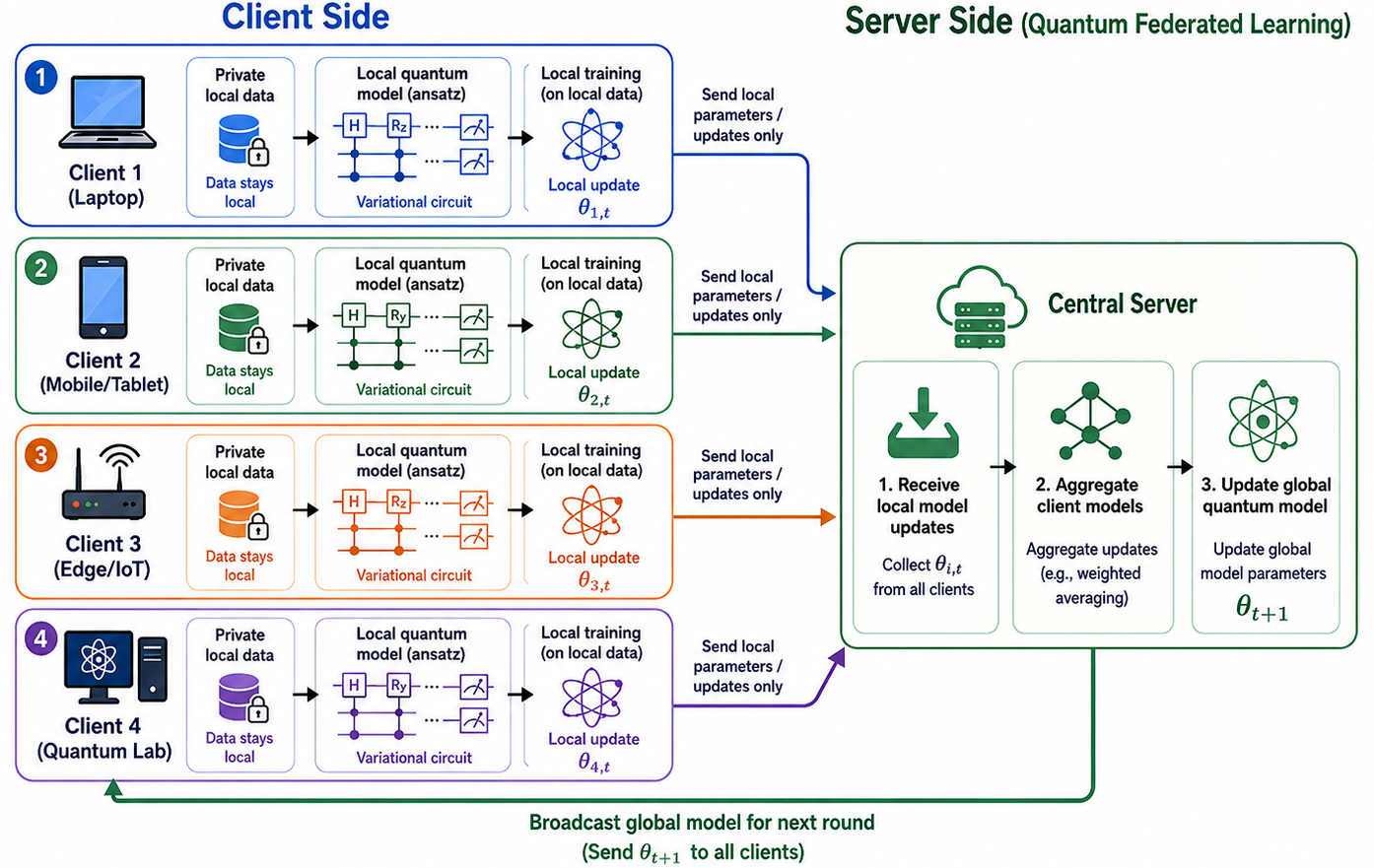}
    \caption{Abstract view of Quantum Federated Learning}
    \label{fig:figg2}
\end{figure}

A second and more QNN-specific challenge arises from the geometry of variational
quantum circuit parameters. Many trainable QNN parameters are rotation angles,
and therefore satisfy the periodic equivalence $\theta \equiv \theta + 2\pi$.
Therefore, parameters close to \(+\pi\) and
\(-\pi\) are geometrically adjacent on the quantum parameter circle, even though
their Euclidean coordinates appear far apart. Naive Euclidean averaging can
therefore move the aggregate toward an artificial value near zero, whereas
circular aggregation preserves the correct angular geometry~\cite{Mardia2000Directional,Fisher1993Circular} and enables adaptive aggregation~\cite{NanayakkaraA2G}.

%\subsection{Motivating Example: Euclidean Seam Artefact}
%\label{subsec:euclidean_seam_artefact}

\subsection{QFL Aggregation Problem and Ideas}
Fig.~\ref{fig:angular_seam_motivation} illustrates the angular aggregation
problem using a single-parameter slice of the QNN parameter space. Although a
full QNN model contains many rotation parameters, the one-dimensional example
captures the essential issue. The ordinary Euclidean mean maps seam-adjacent
client angles to \(14.35^\circ\), whereas the circular mean remains near
\(176.4^\circ\), which is geometrically consistent with the client cluster
around the \(-180^\circ/+180^\circ\) seam.

To this end, we introduce the following three big ideas. We develop Adaptive Aggregation with two Gains (A2G)~\footnote{A preliminary version of A2G is presented in the International Conference on Quantum Communications, Networking, and Computing (QCNC 2026)~\cite{NanayakkaraA2G}.} as a novel framework that uses client reliability and update-control gains to compute a stable global model update in quantum federated learning.

$\circ$\;\textit{QoS-weighted aggregation:} 
QoS-weighted aggregation gives more importance to clients with more reliable updates. 
The client weight depends on data size, quantum fidelity, latency, and update stability. 
Reliable clients contribute more to the global model, while noisy or delayed clients contribute less.

$\circ$\;\textit{Circular A2G:} 
Circular A2G is an adaptive aggregation method for quantum model parameters that are angles. 
It uses circular geometry instead of ordinary Euclidean averaging. 
This avoids wrong updates when angles are close to the $-\pi/\pi$ wrap-around boundary.

$\circ$\;\textit{MP-A2G:} 
MP-A2G means midpoint-projected A2G. 
It first computes a geometry-aware
aggregation direction and then applies a one-shot midpoint-based correction
before updating the server model. This aims to reduce overly aggressive movement
toward an unstable aggregate. We discuss details in Sec~III later.% but it does not iteratively enforce midpoint self-consistency

Importantly, the A2G circular update, midpoint-projected A2G update, and
SCM-A2G update do not collapse toward the misleading Euclidean aggregate.
Instead, starting from the current global point
\(\theta_t=-170^\circ\), they move conservatively toward the
geometry-consistent region of the angular parameter space. This behaviour is
essential: it shows that the proposed updates respect the circular topology of
the quantum parameter manifold and avoid artificial averaging artifacts caused
by treating angles as ordinary Euclidean scalars.

It should be noted that our idea of angular discrepancy has a direct hardware-level interpretation, as shown
in Table~\ref{tab:ibm_geometry_validation_compact}. In our IBM quantum hardware
validation, each aggregated angle is encoded into a single-qubit
\(R_y(\theta)\) circuit and evaluated using the Pauli-\(Z\) expectation value.
The Euclidean mean yields a positive hardware measurement, indicating that it
realizes a physically different quantum state from the geometry-consistent
solutions. In contrast, the circular, midpoint-projected, and SCM-based updates
produce negative IBM hardware measurements that closely match the corresponding
simulator values. This demonstrates that the choice of aggregation geometry is
not merely a mathematical detail; it directly determines the quantum state
implemented on real hardware.

Therefore, geometry-aware consistent aggregation is both theoretically necessary
and experimentally meaningful. By preserving the intrinsic angular structure of
the parameter space, the proposed A2G and SCM-A2G updates avoid spurious
Euclidean averaging, maintain consistency with the intended quantum evolution,
and produce hardware-realized states that align with the true circular geometry
of the model.

\subsection{Why Midpoint Self-Consistency is Needed}
\label{subsec:midpoint_self_consistency}

\begin{figure}[t]
    \centering
    \includegraphics[width=1.0\linewidth]{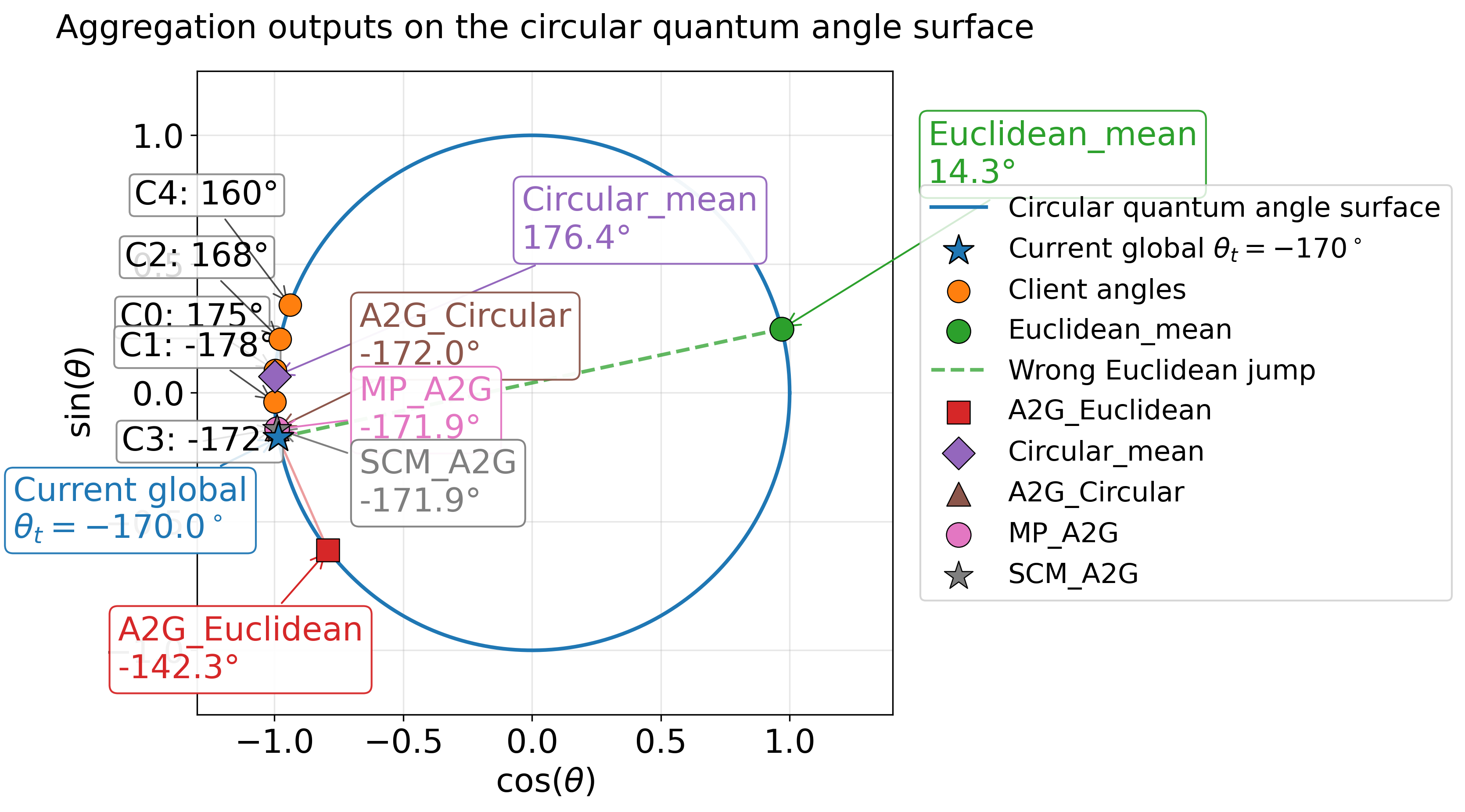}
    \caption{Motivating angular-seam example for QNN parameter aggregation. 
The current global parameter is $\theta_t=-170^\circ$, while client angles lie near the $-180^\circ/+180^\circ$ seam. 
Euclidean averaging maps these seam-adjacent angles to $14.35^\circ$, incorrectly placing the aggregate on the opposite side of the circle. 
In contrast, circular aggregation, midpoint-projected A2G, and SCM-A2G remain in the geometry-consistent angular region.}
    \label{fig:angular_seam_motivation}
\end{figure}
\begin{table}[htbp]
\centering
\scriptsize
\caption{IBM hardware validation of geometry-aware aggregation updates on \texttt{ibm\_fez}.}
\label{tab:ibm_geometry_validation_compact}
\resizebox{\columnwidth}{!}{
\begin{tabular}{lrrrr}
\hline
\textbf{Method} & 
\textbf{$\theta$} & 
\textbf{$\Delta\theta$} & 
\textbf{$Z_{\mathrm{sim}}$} & 
\textbf{$Z_{\mathrm{IBM}}$} \\
\hline
Euclidean mean  & 14.3500   & 184.3500 & 0.968800  & 0.954626 \\
A2G-Euclidean   & -142.3475 & 27.6525  & -0.791730 & -0.796451 \\
Circular mean   & 176.3660  & -13.6340 & -0.997989 & -0.999747 \\
A2G-Circular    & -172.0451 & -2.0451  & -0.990377 & -0.985044 \\
MP-A2G          & -171.8939 & -1.8939  & -0.990009 & -0.987072 \\
SCM-A2G         & -171.9048 & -1.9048  & -0.990035 & -0.985044 \\
\hline
\end{tabular}
}
\vspace{1mm}
\begin{flushleft}
\footnotesize
$\theta$ and $\Delta\theta$ are reported in degrees.
\end{flushleft}
\end{table}
Circular aggregation corrects the angular seam artefact, but it still computes
the client-supported direction mainly from the current global model. In noisy
and heterogeneous QFL, this first direction may be affected by non-IID client
updates, stochastic local QNN training, finite-shot noise, and QoS variation. A
direction that appears suitable at the current point may become less reliable
along the induced movement.

To address this issue, as shown in the Figure \ref{fig:SCM_midpointProcess} midpoint-projected aggregation introduces an intermediate
check. Before accepting the final server update, the server evaluates a
midpoint along the candidate movement and recomputes the client-supported
direction from that midpoint. This midpoint decision acts as a geometry-aware
stability check: it tests whether the proposed movement remains meaningful
after the server has begun to move.

The proposed self-consistent midpoint aggregation further strengthens this idea.
Instead of applying only a one-shot midpoint correction, SCM-A2G accepts a
server movement only when the movement is supported by its own midpoint. Thus,
the next global model is not obtained by direct Euclidean averaging, nor merely
by a circular client mean. It is obtained as a QoS-weighted, torus-aware,
self-consistent movement from the current global model.

\begin{table}[t]
\centering
\scriptsize
\caption{Main abbreviations used in this work.}
\label{tab:main_abbreviations}
\begin{tabular}{l l}
\toprule
\textbf{Abbreviation} & \textbf{Meaning} \\
\midrule
FL & Federated Learning \\
QFL & Quantum Federated Learning \\
QNN & Quantum Neural Network \\
QoS & Quality of Service \\
A2G & Adaptive Aggregation with Two Gains \\
MP-A2G & Midpoint-Projected A2G \\
SCM-A2G & Self-Consistent Midpoint A2G \\
FedAvg & Federated Averaging \\
FedMRUR & Federated Manifold-Regularized Update Rule \\
FEDCOMPASS & Federated scheduling/coordination baseline \\
BAF & Bank Account Fraud dataset \\
SPSA & Simultaneous Perturbation Stochastic Approximation \\
\bottomrule
\end{tabular}
\end{table}

\begin{figure}[htb]
    \centering
    \includegraphics[width=0.8\linewidth]{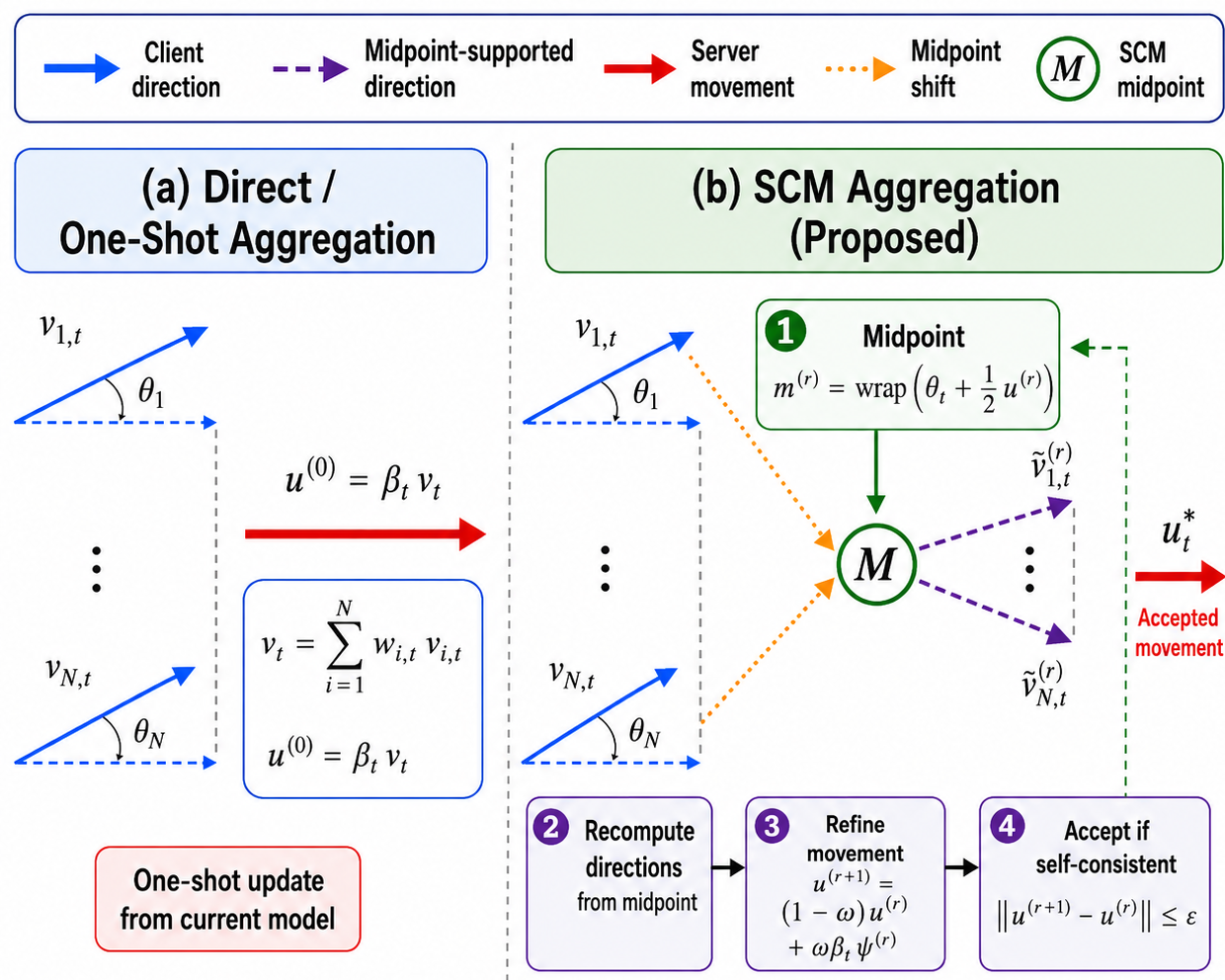}
    \caption{SCM (Self-Consistent Midpoint) Aggregation}
    \label{fig:SCM_midpointProcess}
\end{figure}

\subsection{Key Contributions}
\label{subsec:key_contributions}

Our main contributions are as follows.

\begin{enumerate}
    \item We identify a key aggregation challenge in QFL: client reliability varies, and QNN parameters lie on periodic angular spaces. We show that Euclidean averaging can fail near angular seams and can lead to different quantum observables on IBM hardware.

    \item We propose SCM-A2G-QFL, a QoS-aware and geometry-aware aggregation method with midpoint self-consistency.

\end{enumerate}
We show how FedAvg, QoS-weighted aggregation, circular A2G, MP-A2G, and SCM-A2G fit into one aggregation framework. We evaluate SCM-A2G-QFL using accuracy, validation performance, update norm, volatility, manifold dispersion, and SCM residual.
% The main contributions of this work are summarized as follows.

% \begin{enumerate}
%     \item We identify the coupled reliability--geometry problem in QFL
%     aggregation, where client updates are affected by heterogeneous QoS
%     conditions and QNN parameters lie on circular or torus-valued spaces.

%     \item We show, using a IBM hardware validation, that Euclidean
%     aggregation of seam-adjacent QNN angles can produce a physically different
%     quantum observable, motivating the need for circular and torus-aware
%     aggregation.

%     \item We propose SCM-A2G-QFL, a QoS-weighted self-consistent midpoint
%     aggregation rule that selects the next global model as a movement from the
%     current server STATE whose own midpoint supports the accepted update.

%     \item We establish a unified relationship between FedAvg, QoS-weighted
%     aggregation, circular A2G, midpoint-projected A2G, and the proposed SCM
%     refinement, enabling systematic ablation and comparison.

%     \item We evaluate the proposed method using accuracy, validation-selected
%     performance, update norm, manifold dispersion, volatility, and SCM residual,
%     demonstrating that SCM improves the stability and interpretability of the
%     global aggregation path under noisy and heterogeneous QFL settings.
% \end{enumerate}

\section{Related Work}
\label{sec:limitations_prior_work}

Existing aggregation methods address important parts of federated learning, but
they do not fully resolve the coupled reliability--geometry problem that arises
in QNN-based QFL. Classical FL methods such as FedAvg~\cite{b1},
FedProx~\cite{li2020fedprox}, and SCAFFOLD~\cite{b3}
mainly assume that model parameters can be aggregated in Euclidean space.
FedAvg performs direct weighted model averaging, FedProx introduces a proximal
local objective to reduce client drift, and SCAFFOLD uses control variates to
correct client drift. These methods are effective for many classical FL
settings, but they do not explicitly account for periodic QNN parameters, where
\(\theta \equiv \theta+2\pi\), and therefore may suffer from angular seam
artefacts when applied directly to variational quantum circuits.

Circular statistics and Riemannian averaging provide important foundations for
non-Euclidean aggregation. Circular means avoid angular seam artefacts by
averaging sine--cosine embeddings instead of raw angle values
~\cite{Mardia2000Directional,Fisher1993Circular}. Fréchet/Karcher means
generalize Euclidean averaging to manifold-valued data by minimizing geodesic
distances~\cite{Mancinelli2023}, while Riemannian FL extends federated
optimization to manifold-constrained models~\cite{Li2022RFedSVRG,Huang2026RFedAGS}.
However, these approaches primarily answer the question of where the
geometry-aware mean or manifold optimizer lies. They do not directly define a
QoS-aware server movement from the current global QNN model to the next global
model, nor do they require that the accepted movement remain supported by its
own midpoint.

Midpoint-based correction is also related to numerical methods for
manifold-constrained dynamics. For example, midpoint projection has been used
to improve the stability of stochastic differential equation integration on
manifolds~\cite{joseph2023midpoint}. However, such methods are not
designed for federated aggregation, QoS-weighted client trust, or torus-valued
QNN parameter updates. SCM-A2G-QFL adapts the midpoint principle to the QFL
server-side aggregation setting by making the accepted global movement
self-consistent with its own midpoint.

Recent QFL frameworks further motivate the need for quantum-specific
aggregation. Federated quantum machine learning has been studied in hybrid
classical--quantum settings~\cite{chen2021federated,chehimi2023foundations},
while recent periodic QFL aggregation approaches such as
FEDCOMPASS~\cite{Wang2026FEDCOMPASS} use circular aggregation for quantum
parameters together with client clustering. This confirms that periodicity is
important in QFL. Nevertheless, circular aggregation mainly identifies a
periodic client target; it does not by itself control how far the server should
move from the current global model in one communication round. Similarly,
manifold-regularized FL methods such as FedMRUR~\cite{An2023FedMRUR} address
model inconsistency and update-norm reduction under data heterogeneity, but
they are not designed for QoS-aware torus-valued QNN parameter aggregation.
Ahmad \emph{et al.} studied FL under statistical heterogeneity on
Riemannian manifolds~\cite{ahmad2023federated}, while
FedSPDnet extends geometry-aware FL to SPDNet models on symmetric positive
definite manifolds~\cite{pautrel2026fedspdnet}. These works support the
importance of non-Euclidean FL, but they do not address QoS-aware torus-valued
QNN aggregation or midpoint self-consistent server movement.

In contrast, SCM-A2G-QFL treats the server update itself as the object of
stabilization and update-centric. Circular and Riemannian averaging methods primarily answer where the geometry-aware client mean is located. In contrast, SCM-A2G-QFL asks which movement from the current global model should be accepted under QoS-weighted client evidence and midpoint self-consistency. This distinction is important in QNN-based QFL because a direct aggregate may be geometrically valid as an average but still too aggressive or unstable as a server update under non-IID data, quantum noise, and heterogeneous communication quality.
Table~\ref{tab:component_comparison_scm}
summarizes the key limitations of representative methods relative to the
proposed SCM-A2G-QFL framework.
Table~\ref{tab:component_comparison_scm} summarizes the position of the proposed method relative to representative classical FL, manifold-aware, and QFL aggregation approaches.

\begin{table*}[t]
\centering
\scriptsize
\caption{Component-level comparison of existing aggregation, manifold-aware, and quantum federated learning approaches with the proposed SCM-A2G-QFL method.}
\label{tab:component_comparison_scm}
\begin{tabular}{p{2.4cm} c c c c c c p{3.4cm}}
\hline
\textbf{Method} &
\textbf{FL} &
\textbf{QoS-aware} &
\textbf{ geometry} &
\textbf{Midpoint} &
\textbf{Self-consistent} &
\textbf{Main limitation relative to SCM-A2G-QFL} \\
\hline

FedAvg \cite{b1} &
\checkmark & -- & -- & -- & -- & 
euclidean model averaging, unsuitable for periodic QNN parameters. \\

FedProx \cite{li2020fedprox} &
\checkmark & -- & -- & -- & -- & 
Does not solve angular artifacts in server aggregation. \\

SCAFFOLD \cite{b3}  &
\checkmark & -- & -- & -- & -- & 
Rremains Euclidean and not QNN-angle specific. \\

Circular mean \cite{Mardia2000Directional, Fisher1993Circular} &
-- & -- & \checkmark & -- & -- & 
No server movement control. \\

Fréchet/Karcher \cite{Mancinelli2023} &
-- & optional & \checkmark & implicit & -- & 
No self-consistent current-to-next global update. \\

Reiman-FL \cite{Huang2026RFedAGS, ahmad2023federated} &
\checkmark & -- & \checkmark & -- & -- & 
Not specialized for QNN torus parameters or quantum fidelity/QoS telemetry. \\

%Neurips paper
FedMRUR \cite{An2023FedMRUR} &
\checkmark & -- & \checkmark & -- & -- & 
Uses manifold fusion for classical FL, but not periodic aggregation. \\

FEDCOMPASS \cite{Wang2026FEDCOMPASS} &
\checkmark & partial & \checkmark & -- & -- & 
Lacks movement control. \\

FedSPDnet~\cite{pautrel2026fedspdnet} &
\checkmark & -- & \checkmark & -- & -- &
No torus-valued QNN parameters angular validation. \\

 A2G &
\checkmark & \checkmark & \checkmark & -- &  partial &
Checks one point aggregation direction  \\

\textbf{Ours} &
\checkmark & \checkmark & \checkmark & \checkmark & \checkmark & 
QoS-weighted torus update, midpoint supports the accepted movement. \\
\hline
\end{tabular}
\end{table*}

\begin{figure}[hbt]
     \centering    \includegraphics[width=0.8\linewidth]{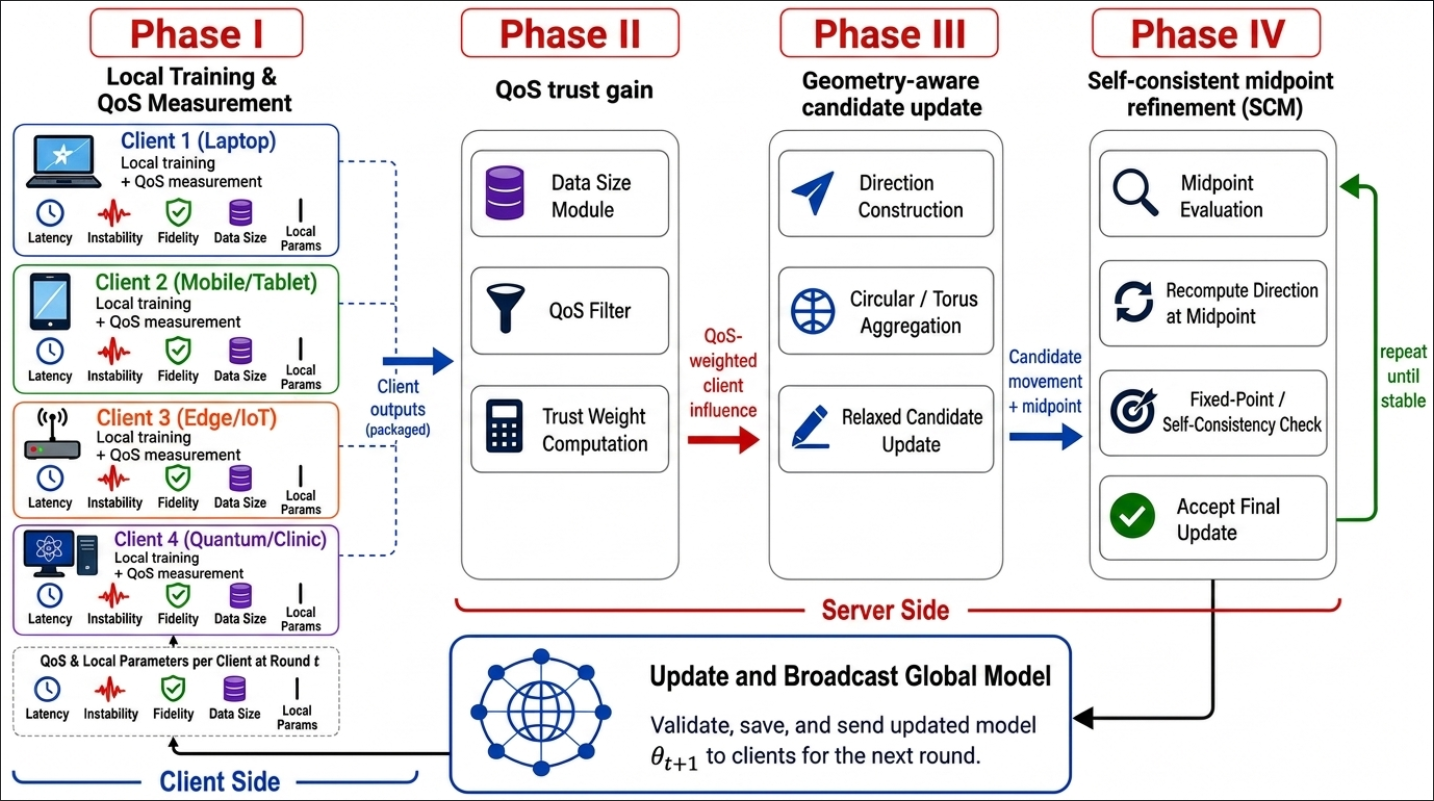}
  \caption{Proposed SCM-A2G-QFL aggregation pipeline.}
\label{fig} 
\end{figure}

\begin{table}[t]
\centering
\scriptsize
\caption{Key notation.}
\label{tab:notation_compact}
\begin{tabular}{l l}
\toprule
\textbf{Notation} & \textbf{Meaning} \\
\midrule
\(t\) & Communication round index \\
\(i\) & Client index \\
\(N\) & Number of clients \\
\(\theta_t\) & Current global angular parameter \\
\(\theta_i\) & Local client angular parameter \\
\(\theta_{t+1}\) & Updated server parameter \\
\(\boldsymbol{\theta}\) & QNN parameter vector \\
\(\Theta\) & Torus-valued angular parameter space \\
\(\wrap(\cdot)\) & Principal-angle wrapping operator \\
\(d_{\mathbb{S}^1}\) & Circular distance on the unit circle \\
\(\bar{\theta}_{\mathrm{Euc}}\) & Euclidean aggregation target \\
\(\bar{\theta}_{\mathrm{circ}}\) & Circular aggregation target \\
\(\theta_{\mathrm{MP}}\) & Midpoint-projected server update \\
\(\theta_{\mathrm{SCM}}\) & Self-consistent midpoint server update \\
\(\alpha\) & QoS sensitivity gain \\
\(\beta\) & Geometry relaxation gain \\
\(q_{i,t}\) & Client QoS score \\
\(w_{i,t}\) & Normalized aggregation weight \\
\(F_{i,t}\) & Fidelity-based quality signal \\
\(L_{i,t}\) & Latency signal \\
\(V_{i,t}\) & Instability signal \\
\(R_{\mathrm{SCM}}\) & SCM fixed-point residual \\
\(\epsilon\) & Numerical stability constant \\
\bottomrule
\end{tabular}
\end{table}
\section{Proposed SCM-A2G-QFL Method}
\label{sec:proposed_method}

The SCM-A2G-QFL framework is developed specifically for quantum federated learning.
We addresses two interconnected challenges: heterogeneous client reliability and the non-Euclidean geometry of quantum parameters.
At each communication
round, clients locally train QNN models and return model parameters together
with QoS-related indicators.
The server then computes reliability-aware client weights, constructs a geometry-aware candidate direction, and refines the accepted global update using midpoint self-consistency.

Let \(\boldsymbol{\theta}_t \in \mathbb{T}^d\) denote the global QNN parameter
vector at communication round \(t\), where \(\mathbb{T}^d\) denotes the
\(d\)-dimensional torus induced by periodic quantum rotation parameters. Let
\(\boldsymbol{\theta}_{i,t}\) be the local model returned by client \(i\), and
let \(p_i=|D_i|/\sum_j |D_j|\) denote the data-size prior of client \(i\). The
goal is to compute the next global parameter vector
\(\boldsymbol{\theta}_{t+1}\) as a controlled, QoS-weighted, geometry-aware
movement from \(\boldsymbol{\theta}_t\).

The first component of SCM-A2G-QFL is the QoS trust gain. In heterogeneous QFL,
not all client updates should contribute equally to the server aggregation. A
client may have a large local dataset but poor quantum-channel reliability, high
communication latency, or unstable local model behaviour. Therefore, before
constructing the global update direction, the server assigns each client a
QoS-aware trust weight.

For client \(i\) at communication round \(t\), let \(F_{i,t}\) denote the
estimated teleportation fidelity or quantum-channel reliability, \(\tau_{i,t}\)
denote the communication latency, and \(V_{i,t}\) denote an instability measure,
such as the variance of the fidelity estimate or the local model variation. We
define the QoS score as
\begin{equation}
q_{i,t}
=
\frac{F_{i,t}^{\alpha}}
{(\tau_{i,t}+\epsilon)^{\gamma}(V_{i,t}+\epsilon)^{\delta}},
\label{eq:qos_score}
\end{equation}
where \(\alpha\) controls the fidelity gain, \(\gamma\) controls the latency
penalty, \(\delta\) controls the instability penalty, and \(\epsilon>0\) avoids
division by zero.

Let
\begin{equation}
p_i
=
\frac{|D_i|}
{\sum_{j=1}^{N}|D_j|}
\label{eq:data_size_prior}
\end{equation}
denote the data-size prior of client \(i\), where \(|D_i|\) is the number of
local training samples. The final normalized aggregation weight is then
\begin{equation}
w_{i,t}
=
\frac{p_i q_{i,t}}
{\sum_{j=1}^{N}p_j q_{j,t}}.
\label{eq:qos_weight}
\end{equation}

Thus, clients with high channel fidelity, low latency, stable local behaviour,
and sufficient data support receive larger aggregation influence. Conversely,
clients with unreliable communication or unstable updates are down-weighted
before the geometry-aware and SCM midpoint-refinement stages.

\subsection{Geometry-Aware Candidate Update}
\label{subsec:geometry_candidate_update}

After computing the QoS-aware weights, the server constructs a geometry-aware
candidate movement. In Euclidean FL, the server can directly average client
parameter vectors. However, QNN parameters are often rotation angles and
therefore satisfy the periodic equivalence
\(\theta \equiv \theta+2\pi\). Consequently, ordinary Euclidean subtraction can
be misleading near the \(-\pi/\pi\) angular seam. To preserve the torus geometry
of the QNN parameter space, we define the local client direction using the
wrapped angular difference:
\begin{equation}
\mathbf{v}_{i,t}
=
\operatorname{wrap}_{[-\pi,\pi)}
\left(
\boldsymbol{\theta}_{i,t}
-
\boldsymbol{\theta}_{t}
\right),
\label{eq:local_tangent_direction}
\end{equation}
where \(\boldsymbol{\theta}_{t}\) is the current global QNN parameter vector and
\(\boldsymbol{\theta}_{i,t}\) is the locally trained parameter vector returned
by client \(i\) at round \(t\).

The QoS-weighted global direction is then computed as
\begin{equation}
\mathbf{v}_{t}
=
\sum_{i=1}^{N}
w_{i,t}\mathbf{v}_{i,t},
\label{eq:weighted_tangent_direction}
\end{equation}
where \(w_{i,t}\) is the normalized QoS trust weight from
Eq.~\eqref{eq:qos_weight}. The initial geometry-aware candidate movement is
\begin{equation}
\mathbf{u}_{t}^{(0)}
=
\beta_t \mathbf{v}_{t},
\label{eq:initial_candidate_movement}
\end{equation}
where \(\beta_t\in(0,1]\) is the geometry gain controlling the size of the
server movement.

A one-shot circular A2G update would therefore be
\begin{equation}
\boldsymbol{\theta}_{t+1}^{\mathrm{A2G}}
=
\operatorname{wrap}_{[-\pi,\pi)}
\left(
\boldsymbol{\theta}_{t}
+
\mathbf{u}_{t}^{(0)}
\right)
=
\operatorname{wrap}_{[-\pi,\pi)}
\left(
\boldsymbol{\theta}_{t}
+
\beta_t\mathbf{v}_{t}
\right).
\label{eq:a2g_circular_update}
\end{equation}
This update respects angular periodicity and avoids Euclidean seam artefacts.
However, the direction \(\mathbf{v}_{t}\) is computed only from the current
global point \(\boldsymbol{\theta}_{t}\). In noisy and heterogeneous QFL, this
initial direction may no longer be sufficiently supported once the server begins
to move. This motivates the midpoint and self-consistency refinement introduced
next.

\subsection{Self-Consistent Midpoint Refinement}
\label{subsec:self_consistent_midpoint}

The midpoint refinement is introduced to avoid accepting a server movement that
is supported only at the current global point but becomes unreliable along the
movement path. 
The one-shot A2G movement \(\mathbf{u}_t^{(0)}=\beta_t\mathbf{v}_t\)
is used only as the initial candidate. SCM then treats the server movement as
a variable \(\mathbf{u}\) and refines it until the movement is supported by its
own midpoint. Given any candidate movement \(\mathbf{u}\), SCM defines the
induced midpoint as

\begin{equation}
\mathbf{m}_t(\mathbf{u})
=
\operatorname{wrap}_{[-\pi,\pi)}
\left(
\boldsymbol{\theta}_t
+
\frac{1}{2}\mathbf{u}
\right).
\label{eq:scm_midpoint}
\end{equation}

From this midpoint, the server recomputes the wrapped client directions:
\begin{equation}
\tilde{\mathbf{v}}_{i,t}(\mathbf{u})
=
\operatorname{wrap}_{[-\pi,\pi)}
\left(
\boldsymbol{\theta}_{i,t}
-
\mathbf{m}_t(\mathbf{u})
\right),
\label{eq:midpoint_client_direction}
\end{equation}
and forms the QoS-weighted midpoint-supported direction
\begin{equation}
\boldsymbol{\psi}_t(\mathbf{u})
=
\sum_{i=1}^{N}
w_{i,t}
\tilde{\mathbf{v}}_{i,t}(\mathbf{u}).
\label{eq:midpoint_direction_map}
\end{equation}
The map \(\boldsymbol{\psi}_t(\mathbf{u})\) may differ from the initial
direction \(\mathbf{v}_t\), because the client directions are evaluated from
the midpoint induced by \(\mathbf{u}\), rather than directly from
\(\boldsymbol{\theta}_t\).

Here, \(\mathbf{u}\) denotes a generic candidate server movement. The final
SCM-accepted movement is denoted by \(\mathbf{u}_t^\star\), which is the
candidate movement satisfying the midpoint fixed-point condition.
\begin{equation}
\mathbf{u}_t^\star
=
\beta_t
\boldsymbol{\psi}_t(\mathbf{u}_t^\star).
\label{eq:scm_fixed_point}
\end{equation}
Thus, the accepted server movement is not merely the first QoS-weighted
direction computed at \(\boldsymbol{\theta}_t\). Instead, it is a movement whose own midpoint continues to support the accepted
movement. The next global
model is then updated as
\begin{equation}
\boldsymbol{\theta}_{t+1}
=
\operatorname{wrap}_{[-\pi,\pi)}
\left(
\boldsymbol{\theta}_{t}
+
\mathbf{u}_t^\star
\right).
\label{eq:scm_global_update}
\end{equation}

In practice, \(\mathbf{u}_t^\star\) is approximated by an under-relaxed
fixed-point iteration. Starting from the geometry-aware candidate movement
\(\mathbf{u}_t^{(0)}=\beta_t\mathbf{v}_t\), SCM iterates
\begin{equation}
\mathbf{u}_t^{(r+1)}
=
(1-\omega)\mathbf{u}_t^{(r)}
+
\omega\beta_t
\boldsymbol{\psi}_t(\mathbf{u}_t^{(r)}),
\label{eq:scm_under_relaxed_iteration}
\end{equation}
where \(\omega\in(0,1]\) is the solver relaxation factor. If the midpoint
recomputed direction disagrees with the current candidate movement, the
iteration adjusts the movement through
\(\boldsymbol{\psi}_t(\mathbf{u}_t^{(r)})\). Hence, SCM changes the accepted
movement according to the QoS-weighted client evidence observed from the
midpoint, rather than accepting the one-shot direction.

The iteration stops when
\begin{equation}
\left\|
\operatorname{wrap}_{[-\pi,\pi)}
\left(
\mathbf{u}_t^{(r+1)}-\mathbf{u}_t^{(r)}
\right)
\right\|_2
<
\varepsilon_{\mathrm{SCM}},
\label{eq:scm_stopping}
\end{equation}
or when a maximum number of SCM iterations is reached. The final iterate is then
used as the accepted movement \(\mathbf{u}_t^\star\).

We record the SCM residual as
\begin{equation}
R_{\mathrm{SCM},t}
=
\left\|
\operatorname{wrap}_{[-\pi,\pi)}
\left(
\mathbf{u}_t^\star
-
\beta_t\boldsymbol{\psi}_t(\mathbf{u}_t^\star)
\right)
\right\|_2,
\label{eq:scm_residual}
\end{equation}
which measures how closely the accepted movement satisfies the midpoint
self-consistency condition.

\begin{algorithm}[t]
\caption{QoS Trust Weight Computation}
\label{alg:qos_trust}
\begin{algorithmic}[1]
\REQUIRE Selected clients \(\mathcal{S}_t\); local dataset sizes
\(\{|D_i|\}_{i\in\mathcal{S}_t}\); QoS indicators
\(\{F_{i,t},\tau_{i,t},V_{i,t}\}_{i\in\mathcal{S}_t}\); QoS gains
\(\alpha,\gamma,\delta\); numerical constant \(\epsilon>0\)
\ENSURE Normalized QoS-aware weights \(\{w_{i,t}\}_{i\in\mathcal{S}_t}\)

\FOR{each selected client \(i\in\mathcal{S}_t\)}
    \STATE Compute data-size prior:$
    p_{i,t}
    =
    \frac{|D_i|}
    {\sum_{j\in\mathcal{S}_t}|D_j|}.$
    \STATE Compute QoS score:$
    q_{i,t}
    =
    \frac{F_{i,t}^{\alpha}}
    {(\tau_{i,t}+\epsilon)^{\gamma}(V_{i,t}+\epsilon)^{\delta}}.$
    \STATE Compute unnormalized trust:$
    a_{i,t}=p_{i,t}q_{i,t}.$
\ENDFOR

\FOR{each selected client \(i\in\mathcal{S}_t\)}
    \STATE Normalize aggregation weight:$
    w_{i,t}
    =
    \frac{a_{i,t}}
    {\sum_{j\in\mathcal{S}_t}a_{j,t}}.$
\ENDFOR
\STATE \textbf{return} \(\{w_{i,t}\}_{i\in\mathcal{S}_t}\)
\end{algorithmic}
\end{algorithm}

\begin{algorithm}[t]
\caption{SCM-A2G Server Update}
\label{alg:scm_a2g_update}
\begin{algorithmic}[1]
\REQUIRE Current global model \(\boldsymbol{\theta}_t\); selected client models
\(\{\boldsymbol{\theta}_{i,t}\}_{i\in\mathcal{S}_t}\); QoS weights
\(\{w_{i,t}\}_{i\in\mathcal{S}_t}\); geometry gain \(\beta_t\); relaxation
\(\omega\); maximum SCM iterations \(R_{\max}\); tolerance
\(\varepsilon_{\mathrm{SCM}}\)
\ENSURE Next global model \(\boldsymbol{\theta}_{t+1}\)

\FOR{each selected client \(i\in\mathcal{S}_t\)}
    \STATE Compute local wrapped direction:
    $\mathbf{v}_{i,t}
    =
    \operatorname{wrap}_{[-\pi,\pi)}
    \left(
    \boldsymbol{\theta}_{i,t}
    -
    \boldsymbol{\theta}_t
    \right).$
\ENDFOR

\STATE Compute QoS-weighted direction:
$
\mathbf{v}_t
=
\sum_{i\in\mathcal{S}_t}w_{i,t}\mathbf{v}_{i,t}.$

\STATE Initialize SCM movement:
$
\mathbf{u}_t^{(0)}=\beta_t\mathbf{v}_t.
$
\FOR{\(r=0,\ldots,R_{\max}-1\)}
    \STATE Compute midpoint:$
    \mathbf{m}_t^{(r)}
    =
    \operatorname{wrap}_{[-\pi,\pi)}
    \left(
    \boldsymbol{\theta}_t+\frac{1}{2}\mathbf{u}_t^{(r)}
    \right).$

    \FOR{each selected client \(i\in\mathcal{S}_t\)}
        \STATE Recompute direction from midpoint:$
        \tilde{\mathbf{v}}_{i,t}^{(r)}
        =
        \operatorname{wrap}_{[-\pi,\pi)}
        \left(
        \boldsymbol{\theta}_{i,t}
        -
        \mathbf{m}_t^{(r)}
        \right).$
    \ENDFOR

    \STATE Compute midpoint-supported direction:$
    \boldsymbol{\psi}_t^{(r)}
    =
    \sum_{i\in\mathcal{S}_t}
    w_{i,t}
    \tilde{\mathbf{v}}_{i,t}^{(r)}.$

    \STATE Compute fixed-point candidate:$
    \mathbf{u}_{\mathrm{fp},t}^{(r)}
    =
    \beta_t\boldsymbol{\psi}_t^{(r)}.$

    \STATE Apply under-relaxed refinement:$
    \mathbf{u}_t^{(r+1)}
    =
    (1-\omega)\mathbf{u}_t^{(r)}
    +
    \omega\mathbf{u}_{\mathrm{fp},t}^{(r)}.$
    \STATE Compute fixed-point change:$
    \Delta_t^{(r)}
    =
    \left\|
    \operatorname{wrap}_{[-\pi,\pi)}
    \left(
    \mathbf{u}_t^{(r+1)}-\mathbf{u}_t^{(r)}
    \right)
    \right\|_2.$
    \IF{\(\Delta_t^{(r)}<\varepsilon_{\mathrm{SCM}}\)}
        \STATE \textbf{break}
    \ENDIF
\ENDFOR

\STATE Set accepted movement:
$
\mathbf{u}_t^\star=\mathbf{u}_t^{(r+1)}.$

\STATE Update global model:
$
\boldsymbol{\theta}_{t+1}
=
\operatorname{wrap}_{[-\pi,\pi)}
\left(
\boldsymbol{\theta}_t+\mathbf{u}_t^\star
\right).$
\STATE \textbf{return} \(\boldsymbol{\theta}_{t+1}\)
\end{algorithmic}
\end{algorithm}

\begin{algorithm}[t]
\caption{SCM-A2G-QFL Training Procedure}
\label{alg:scm_a2g_qfl_training}
\begin{algorithmic}[1]
\REQUIRE Initial global QNN parameters \(\boldsymbol{\theta}_0\); clients
\(\{1,\ldots,N\}\); communication rounds \(T\); local training routine
\(\mathcal{L}\); QoS gains \(\alpha,\gamma,\delta\); geometry gain
\(\beta_t\); SCM parameters \(R_{\max},\omega,\varepsilon_{\mathrm{SCM}}\)
\ENSURE Final global model \(\boldsymbol{\theta}_T\)

\FOR{\(t=0,\ldots,T-1\)}
    \STATE Select participating clients \(\mathcal{S}_t\subseteq\{1,\ldots,N\}\).
    \STATE Server broadcasts \(\boldsymbol{\theta}_t\) to clients in \(\mathcal{S}_t\).

    \FOR{each selected client \(i\in\mathcal{S}_t\) in parallel}
        \STATE Initialize local QNN with \(\boldsymbol{\theta}_t\).
        \STATE Train locally using routine \(\mathcal{L}\) on private data \(D_i\).
        \STATE Return local model \(\boldsymbol{\theta}_{i,t}\).
        \STATE Estimate QoS indicators \(F_{i,t}\), \(\tau_{i,t}\), and \(V_{i,t}\).
    \ENDFOR

    \STATE Compute QoS-aware weights \(\{w_{i,t}\}_{i\in\mathcal{S}_t}\) using Algorithm~\ref{alg:qos_trust}.

    \STATE Compute \(\boldsymbol{\theta}_{t+1}\) using Algorithm~\ref{alg:scm_a2g_update}.

    \STATE Record diagnostics: global accuracy, validation loss, update norm, manifold dispersion, weight entropy, and SCM residual.
\ENDFOR

\STATE \textbf{return} \(\boldsymbol{\theta}_T\)
\end{algorithmic}
\end{algorithm}

\section{Convergence Analysis}
\label{sec:convergence_analysis}

We analyze the convergence behaviour of SCM-A2G-QFL under a
general non-convex federated objective. 
Detailed proofs are provided in the
Appendix. 
The analysis shows that the proposed update behaves as a
geometry-aware descent step whose error is controlled by client heterogeneity,
QoS-weight variance, angular wrapping error, and the SCM fixed-point residual.
\subsection{Assumptions}
\label{subsec:convergence_assumptions}

Let the global objective be
\begin{equation}
F(\boldsymbol{\theta})
=
\sum_{i=1}^{N}p_i F_i(\boldsymbol{\theta}),
\qquad
p_i=\frac{|D_i|}{\sum_{j=1}^{N}|D_j|},
\label{eq:global_objective}
\end{equation}
where \(F_i\) is the local objective of client \(i\), and
\(\boldsymbol{\theta}\in\mathbb{T}^d\) denotes the torus-valued QNN parameter
vector. For analysis, we work in a local tangent chart induced by the wrapped
angular difference \(\operatorname{wrap}_{[-\pi,\pi)}(\cdot)\).

\begin{assumption}
    Smoothness: Each local objective \(F_i\) and the global objective \(F\) are \(L\)-smooth in
the local tangent chart induced by the wrapped angular difference
\cite{li2020fedprox}., i.e.,
\begin{equation}
\|\nabla F_i(\boldsymbol{a})-\nabla F_i(\boldsymbol{b})\|
\leq
L\|\operatorname{wrap}(\boldsymbol{a}-\boldsymbol{b})\|,
\end{equation}
and similarly for \(F\).
\end{assumption}

\begin{assumption}
    {Unbiased local stochastic gradients: }
For client \(i\), the stochastic gradient estimator
\(\mathbf{g}_{i,t}\) satisfies
\begin{equation}
\mathbb{E}[\mathbf{g}_{i,t}]
=
\nabla F_i(\boldsymbol{\theta}_t),
\qquad
\mathbb{E}\|\mathbf{g}_{i,t}-\nabla F_i(\boldsymbol{\theta}_t)\|^2
\leq
\sigma_l^2.
\end{equation}
\end{assumption}

\begin{assumption}
{Bounded client heterogeneity:} 
The gradient dissimilarity across clients is bounded, as commonly assumed in
heterogeneous FL analyses~\cite{li2020fedprox,b3}.
\begin{equation}
\mathbb{E}\|\nabla F_i(\boldsymbol{\theta})-\nabla F(\boldsymbol{\theta})\|^2
\leq
\sigma_g^2.
\label{eq:heterogeneity_bound}
\end{equation}
\end{assumption}

\begin{assumption}{Bounded QoS weights:}
The QoS-aware aggregation weights satisfy $
w_{i,t}\geq 0,\qquad
\sum_{i=1}^{N}w_{i,t}=1,\qquad
w_{i,t}\leq w_{\max}.$
Moreover, the deviation between QoS weights and data-size priors is bounded:
\begin{equation}
\sum_{i=1}^{N}|w_{i,t}-p_i|
\leq
\rho_q.
\label{eq:qos_weight_deviation}
\end{equation}
\end{assumption}

\begin{assumption}[Bounded angular movement and SCM residual contribution]
\label{assump:bounded_movement_residual}
The wrapped local client directions and accepted SCM server movements are
bounded:
\begin{equation}
\|\mathbf{v}_{i,t}\|\leq G,
\qquad
\|\mathbf{u}_t^\star\|\leq U .
\label{eq:bounded_movement}
\end{equation}
Moreover, the SCM fixed-point residual contributes a bounded descent-scale
error:
\begin{equation}
\mathbb{E}\!\left[
\left\langle
\nabla F(\boldsymbol{\theta}_t),
\mathbf{e}_{\mathrm{SCM},t}
\right\rangle
\right]
\leq
\frac{\beta_t}{4}
\mathbb{E}\!\left[
\|\nabla F(\boldsymbol{\theta}_t)\|^2
\right]
+
C_3R_{\mathrm{SCM},t}^2 .
\label{eq:scm_residual_descent_scale}
\end{equation}
Here, \(R_{\mathrm{SCM},t}\) denotes the SCM residual contribution measured at
the descent scale. This condition is encouraged by angular wrapping, bounded
geometry gain \(\beta_t\), and the stopping tolerance used in the SCM
fixed-point solver.
\end{assumption}

\subsection{SCM Update Error Decomposition}
\label{subsec:scm_error_decomposition}

The ideal centralized descent direction at round \(t\) is
\(-\nabla F(\boldsymbol{\theta}_t)\). SCM-A2G instead applies the accepted
server movement
\begin{equation}
\mathbf{u}_t^\star
=
\beta_t \boldsymbol{\psi}_t(\mathbf{u}_t^\star)
+
\mathbf{e}_{\mathrm{SCM},t},
\label{eq:scm_residual_error}
\end{equation}
where
\begin{equation}
\mathbf{e}_{\mathrm{SCM},t}
=
\mathbf{u}_t^\star
-
\beta_t \boldsymbol{\psi}_t(\mathbf{u}_t^\star),
\qquad
\|\mathbf{e}_{\mathrm{SCM},t}\|
=
R_{\mathrm{SCM},t}.
\label{eq:scm_residual_vector}
\end{equation}
In the convergence analysis, \(R_{\mathrm{SCM},t}\) denotes the corresponding
descent-scale residual contribution defined through
Assumption~\ref{assump:bounded_movement_residual}. Thus,
\(R_{\mathrm{SCM},t}\) measures how much the remaining SCM fixed-point error
contributes to the one-step descent inequality.ndition.

The midpoint-supported direction can be decomposed as
\begin{equation}
\boldsymbol{\psi}_t(\mathbf{u}_t^\star)
=
-\nabla F(\boldsymbol{\theta}_t)
+
\boldsymbol{\xi}_{l,t}
+
\boldsymbol{\xi}_{g,t}
+
\boldsymbol{\xi}_{q,t}
+
\boldsymbol{\xi}_{m,t},
\label{eq:scm_direction_decomposition}
\end{equation}
where \(\boldsymbol{\xi}_{l,t}\) denotes stochastic local optimization error,
\(\boldsymbol{\xi}_{g,t}\) denotes client heterogeneity error,
\(\boldsymbol{\xi}_{q,t}\) denotes QoS reweighting bias, and
\(\boldsymbol{\xi}_{m,t}\) denotes midpoint displacement error.

Under
Assumptions~1--5, these terms are bounded in expectation by
\begin{equation}
\begin{aligned}
\mathbb{E}\|\boldsymbol{\xi}_{l,t}\|^2
&\leq C_l\sigma_l^2,
&
\mathbb{E}\|\boldsymbol{\xi}_{g,t}\|^2
&\leq C_g\sigma_g^2, \\
\mathbb{E}\|\boldsymbol{\xi}_{q,t}\|^2
&\leq C_q\rho_q^2G^2,
&
\mathbb{E}\|\boldsymbol{\xi}_{m,t}\|^2
&\leq C_mL^2\|\mathbf{u}_t^\star\|^2 .
\end{aligned}
\label{eq:error_bounds}
\end{equation}

\subsection{Descent Lemma}
\label{subsec:descent_lemma}

\begin{lemma}[One-step descent under SCM-A2G]
\label{lemma:scm_one_step_descent}
Under Assumptions~1--5, if the geometry gain satisfies
\(0<\beta_t\leq 1/L\), then the SCM-A2G update
\begin{equation}
\boldsymbol{\theta}_{t+1}
=
\operatorname{wrap}_{[-\pi,\pi)}
\!\left(\boldsymbol{\theta}_t+\mathbf{u}_t^\star\right)
\label{eq:scm_update_for_descent}
\end{equation}
satisfies
\begin{equation}
\begin{aligned}
\mathbb{E}\!\left[
F(\boldsymbol{\theta}_{t+1})
\right]
\leq\;&
\mathbb{E}\!\left[
F(\boldsymbol{\theta}_{t})
\right]
-
\frac{\beta_t}{2}
\mathbb{E}\!\left[
\left\|\nabla F(\boldsymbol{\theta}_t)\right\|^2
\right] \\
&+
C_1\beta_t
\left(
\sigma_l^2+\sigma_g^2+\rho_q^2G^2
\right) \\
&+
C_2\beta_t L^2
\mathbb{E}\!\left[
\left\|\mathbf{u}_t^\star\right\|^2
\right]
+
C_3 R_{\mathrm{SCM},t}^2 .
\end{aligned}
\label{eq:scm_descent_lemma}
\end{equation}
\end{lemma}

\begin{proof}[Proof sketch]
The proof applies the \(L\)-smoothness inequality in the local tangent chart,
substitutes the SCM update decomposition in
Eq.~\eqref{eq:scm_direction_decomposition}, and bounds the stochastic local
optimization error, client heterogeneity error, QoS-induced weighting error,
midpoint approximation error, and fixed-point residual error using
Assumptions~1--5.
\end{proof}

\subsection{Main Convergence Result}
\label{subsec:main_convergence_result}

\begin{theorem}[Non-convex convergence of SCM-A2G-QFL]
\label{theorem:main_convergence_result}
Suppose Assumptions~1--5 hold and choose a constant geometry gain
\(\beta_t=\beta\) such that \(0<\beta\leq 1/L\). Then, after \(T\)
communication rounds, SCM-A2G-QFL satisfies 
\begin{equation}
\begin{aligned}
\frac{1}{T}\sum_{t=0}^{T-1}&
\mathbb{E}\!\left[
\|\nabla F(\boldsymbol{\theta}_t)\|^2
\right]
\leq\;
\frac{2(F(\boldsymbol{\theta}_0)-F^\star)}{\beta T} \\&
+
C_1(\sigma_l^2+\sigma_g^2+\rho_q^2G^2) +
C_2L^2
\frac{1}{T}\sum_{t=0}^{T-1}
\mathbb{E}\!\left[
\|\mathbf{u}_t^\star\|^2
\right]\\&
+ 
\frac{C_3}{\beta T}
\sum_{t=0}^{T-1}R_{\mathrm{SCM},t}^2 .
\end{aligned}
\label{eq:main_convergence_bound}
\end{equation}

Here \(F^\star\) is a lower bound of the global objective, and
\(C_1,C_2,C_3>0\) are constants independent of \(T\).

If the SCM residual is uniformly bounded by
\(R_{\mathrm{SCM},t}\leq\varepsilon_{\mathrm{SCM}}\), then

\begin{equation}
\begin{aligned}
\frac{1}{T}
\sum_{t=0}^{T-1}
\mathbb{E}\!\left[
\|\nabla F(\boldsymbol{\theta}_t)\|^2
\right]
\leq
&\;
\mathcal{O}\!\left(\frac{1}{\beta T}\right)
+
\mathcal{O}\!\left(\sigma_l^2+\sigma_g^2+\rho_q^2G^2\right) \\
&+
\mathcal{O}\!\left(L^2\overline{U^2}\right)
+
\mathcal{O}\!\left(\frac{\varepsilon_{\mathrm{SCM}}^2}{\beta}\right).
\end{aligned}
\label{eq:scm_simplified_rate}
\end{equation}

where $
\overline{U^2}
=
\frac{1}{T}\sum_{t=0}^{T-1}
\mathbb{E}\|\mathbf{u}_t^\star\|^2.$
\end{theorem} 

\begin{proof}
The proof follows by applying the smoothness descent lemma to the accepted
SCM-A2G server update, decomposing the update error into stochastic local
optimization error, client heterogeneity error, QoS-induced weighting error,
geometry approximation error, and SCM fixed-point residual error. Summing the
resulting one-step descent inequality over \(t=0,\ldots,T-1\), using the lower
boundedness of \(F\), and rearranging terms yields
\eqref{eq:main_convergence_bound}.
\end{proof}

\subsection{Interpretation}
\label{subsec:convergence_interpretation}

Theorem~\ref{theorem:main_convergence_result} shows that SCM-A2G-QFL retains
the standard non-convex FL convergence structure. The average squared gradient
norm decreases at an \(\mathcal{O}(1/T)\) rate up to a neighbourhood determined
by stochastic local training, client heterogeneity, QoS-induced reweighting,
geometry-induced movement error, and the SCM fixed-point residual. The residual
term $
\frac{1}{\beta T}\sum_{t=0}^{T-1}R_{\mathrm{SCM},t}^2$
is specific to the proposed self-consistent midpoint mechanism. It shows that
the accuracy of the midpoint fixed-point solve directly affects the final
convergence neighbourhood.

The result also explains the empirical stability observed in the experiments.
The geometry gain \(\beta\) controls how aggressively the server moves from the
current global parameter \(\boldsymbol{\theta}_t\) toward the QoS-weighted
geometry-aware direction. Smaller accepted movements can reduce the midpoint
displacement term \(L^2\overline{U^2}\), but excessively small \(\beta\) can
increase the relative effect of the residual contribution
\(\varepsilon_{\mathrm{SCM}}^2/\beta\). Thus, \(\beta\) represents a
stability--progress trade-off rather than a purely stabilizing parameter.

Overall, SCM-A2G does not merely change the aggregation target. It controls the
path by which the global QNN model moves from \(\boldsymbol{\theta}_t\) to
\(\boldsymbol{\theta}_{t+1}\), requiring the accepted movement to remain
consistent with the QoS-weighted direction evaluated at its own midpoint. This
provides a theoretical explanation for the reduced late-round volatility
observed in the experiments.

\section{Experimental Evaluation}
\subsection{Datasets and Domain Generalization}
\label{subsec:datasets_domain_generalization}

We use two binary classification datasets from
substantially different domains.
The Breast-Lesions-USG represents a medical
ultrasound diagnosis task, where patient age and ultrasound descriptors are used
to classify lesions as benign or malignant \cite{pawlowska2024curated}. 
Invalid entries are excluded,
categorical features are label-encoded, age is converted to numeric form, and
principal component analysis (PCA) apply after standardization to obtain \(d=2\) or \(d=5\)
QNN-compatible features.

The Bank Account Fraud (BAF) dataset serves as a benchmark for financial fraud detection.
BAF was introduced at NeurIPS 2022 as a privacy-preserving, large-scale resource.
It provides a realistic tabular benchmark for detecting fraud in bank account openings.
\cite{jesus2022turning}. The suite comprises six synthetic datasets generated
from an anonymized real-world fraud detection environment and addresses practical
challenges, including class imbalance, temporal dynamics, and controlled bias
variants. 
Each dataset includes 32 input features and up to one million records.
The binary target variable is \texttt{fraud\_bool}. In the conducted experiments, missing
values are removed, categorical variables are ordinal encoded, and numerical
features are standardized. PCA is applied to obtain \(d=4\) QNN-compatible
features. 
For both datasets, train, validation, and test splits are performed
prior to encoding, scaling, and PCA to prevent data leakage.

\subsection{Experimental Configuration}
\label{subsec:experimental_configuration}

All experiments are implemented in Python using Qiskit, Qiskit Aer, and Qiskit
Machine Learning. Qiskit Aer is used for simulator-based QNN execution. The main
supporting libraries include NumPy, pandas, scikit-learn, and matplotlib. Fixed
random seeds are used for dataset splitting, client partitioning, QNN parameter
initialization, and stochastic optimizer behaviour wherever supported.
The main experimental configuration is summarized in
Table~\ref{tab:experimental_config}. Dataset-specific adjustments are made when
required by sample size, circuit dimension, and computational budget.

\begin{table}[t]
\centering
\caption{Main experimental configuration used in the domain-generalized A2G and SCM-A2G-QFL experiments.}
\label{tab:experimental_config}
\begin{tabular}{ll}
\hline
\textbf{Setting} & \textbf{Value} \\
\hline
Programming language & Python 3.11 \\
Quantum framework & Qiskit, Qiskit Machine Learning \\
Simulator & Qiskit Aer simulator \\
Feature map & ZZFeatureMap \\
Ansatz & RealAmplitudes \\
Task type & Binary classification \\
Datasets & Breast-Lesions-USG, BAF \\
Domains & Medical ultrasound, financial fraud detection \\
QNN input dimension & \(d=2,4,5\) depending on dataset \\
Number of clients & 5 or 10 \\
Federated settings & IID, label-skewed, quantity-skewed \\
Aggregation baselines & FedAvg, A2G, MP-A2G, SCM-A2G-QFL \\
QoS signal & Quantum fidelity \\
Shots & 1024 \\
\hline
\end{tabular}
\end{table}
 
 \subsection{Controlled Angular Stress-Test Validation}
\label{subsec:controlled_angular_stress_test}

\begin{table*}[h]
\centering
\scriptsize
\caption{Controlled angular stress-test cases. Angles are in degrees.
Euclidean aggregation can move to seam-inconsistent regions in wrap-around
cases, whereas MP-A2G and SCM produce bounded geometry-aware server
movements. \(R_{\mathrm{SCM}}\) denotes the final SCM fixed-point residual. Here, \(\theta_{\mathrm{Euc}}\) and \(\theta_{\mathrm{Circ}}\) denote direct aggregation targets, whereas \(\theta_{\mathrm{MP}}\) and \(\theta_{\mathrm{SCM}}\) denote the next server states obtained after applying bounded geometry-aware movement from the current server parameter \(\theta_t\).}
\label{tab:controlled_stress_summary}
\resizebox{\textwidth}{!}{%
\begin{tabular}{l r l r r r r r}
\toprule
\textbf{Case} &
\(\boldsymbol{\theta_t}\) &
\textbf{Client angles} &
\(\boldsymbol{\theta_{\mathrm{Euc}}}\) &
\(\boldsymbol{\theta_{\mathrm{Circ}}}\) &
\(\boldsymbol{\theta_{\mathrm{MP}}}\) &
\(\boldsymbol{\theta_{\mathrm{SCM}}}\) &
\(\boldsymbol{R_{\mathrm{SCM}}}\) \\
\midrule
A: small cluster & -170.0 & \(5.7,-11.5,2.9\) 
& -0.9667 & -0.9528 & -146.5466 & -146.4140 & \(2.0\times10^{-6}\) \\
B: two-client wrap & -170.0 & \(174.3,-174.3\) 
& 0.0000 & -180.0000 & -171.3875 & -171.3953 & \(0.0\) \\
C: symmetric spread & -170.0 & \(-85.9,0.0,85.9\) 
& 0.0000 & 0.0000 & -163.0625 & -163.0233 & \(1.0\times10^{-6}\) \\
D: three-client wrap & -170.0 & \(168.5,-171.4,177.1\) 
& 58.0667 & 178.0617 & -171.6558 & -171.6651 & \(0.0\) \\
\bottomrule
\end{tabular}%
}
\vspace{1mm}
\begin{flushleft}
\footnotesize
\(\theta_{\mathrm{Euc}}\) denotes the Euclidean mean, \(\theta_{\mathrm{Circ}}\) denotes the circular target, \(\theta_{\mathrm{MP}}\) denotes the midpoint-projected update, and \(\theta_{\mathrm{SCM}}\) denotes the proposed self-consistent midpoint update. \(R_{\mathrm{SCM}}\) is the SCM fixed-point residual.
\end{flushleft}
\end{table*}

\begin{figure}[t]
    \centering
    \includegraphics[width=\linewidth]{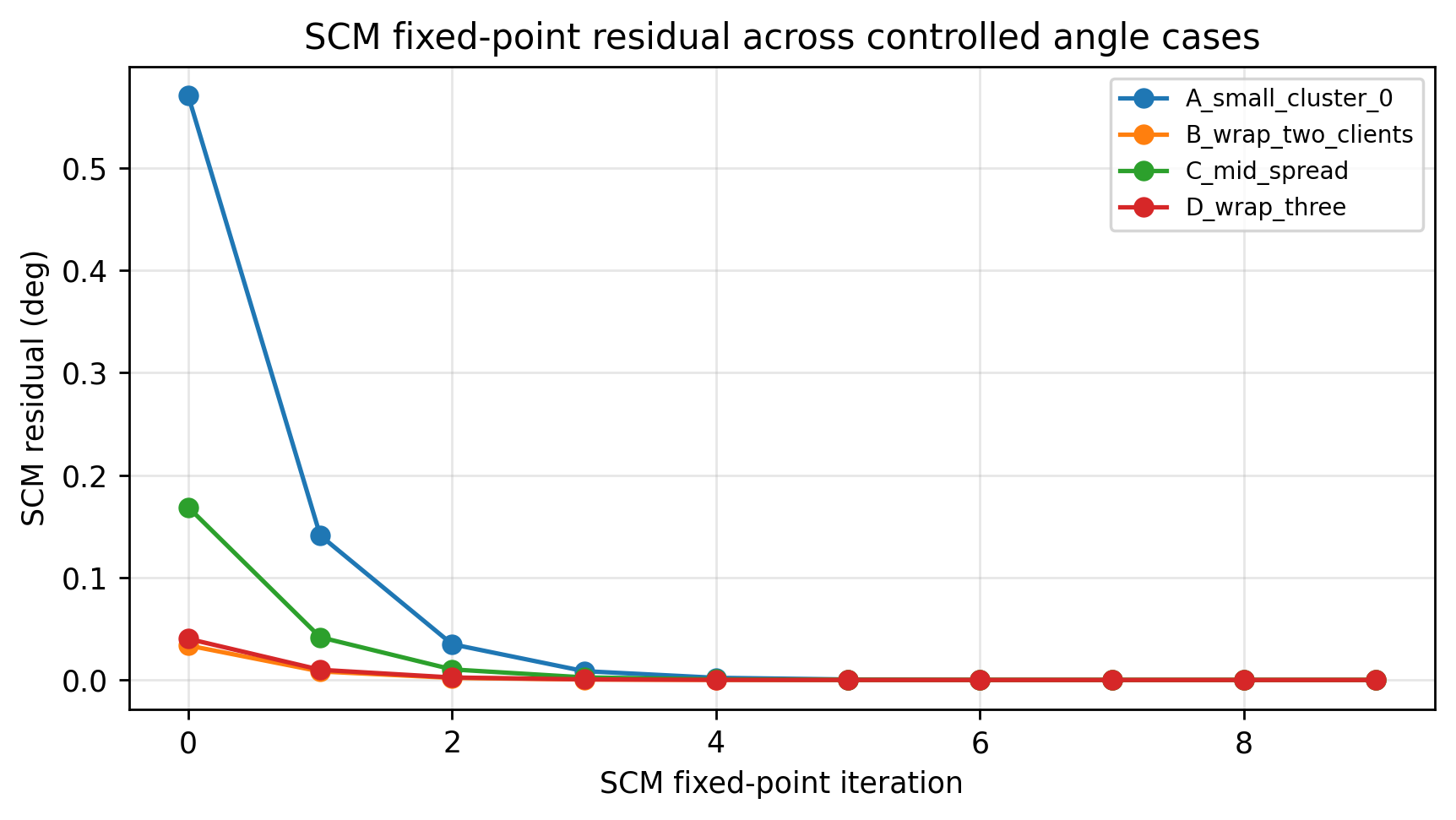}
    \caption{SCM fixed-point residual across controlled angular stress-test cases.
The residual decreases rapidly across SCM iterations, showing that the accepted
server movement becomes self-consistent with its own midpoint. The convergence
remains stable from simple non-wrap cases to more challenging wrap-around
configurations. These resutls are generated in Real Qiskit IBM. Appendix provide more results.}
    \label{fig:scm_residual_controlled_cases}
\end{figure}

Prior to evaluating comprehensive federated QNN training, the angular component is first isolated.
This is achieved by examining aggregation behavior through controlled one-dimensional stress tests. 
These scenarios
represent coordinate-wise QNN rotation parameters on the torus and enable direct assessment of
whether the server update remains geometrically consistent near
angular seams i.e., the wrap-around boundary where \(+\pi\) and \(-\pi\) denote adjacent directions..

Table~\ref{tab:controlled_stress_summary} reports four controlled angular
stress tests. Case A is a non-wrap control case, where the client angles are
clustered away from the \(-180^\circ/+180^\circ\) seam. Euclidean and circular
aggregation therefore produce nearly identical targets, showing that the
geometry-aware formulation does not distort benign angular configurations.

Cases B and D test wrap-around behaviour, where the client angles are
geometrically concentrated near the seam. In these cases, Euclidean averaging
produces seam-inconsistent targets, such as \(0^\circ\) and \(58.07^\circ\),
while circular aggregation recovers the correct seam-consistent target. In
contrast, MP-A2G and SCM-A2G do not directly jump to this target; they compute
bounded server movements from the current global angle. Thus, values near
\(-171^\circ\) in Cases B and D reflect controlled relaxation rather than
aggregation error.

This distinction is central: circular aggregation identifies the
geometry-aware client target, whereas SCM-A2G determines the accepted next
server state. The SCM update is constrained by the geometry gain and midpoint
self-consistency, requiring the accepted movement to remain aligned with the
QoS-weighted direction at its own midpoint. Fig.~\ref{fig:scm_residual_controlled_cases}
confirms this fixed-point behaviour, as the SCM residual rapidly decreases to
near zero in all cases. These tests therefore show that SCM-A2G preserves
benign-case behaviour, avoids seam-induced Euclidean failures, and stabilizes
the server update through midpoint fixed-point refinement.

\subsection{Results}

\subsubsection{Accuracy–stability behaviour}

Fig.~\ref{fig:accuracy_stability} compares the accuracy--stability behaviour
of the methods on Breast-Lesions-USG and BAF. Figures~(a) and~(c) show the
last-5-round mean accuracy, while (b) and~(d) show the corresponding
late-round volatility. On Breast-Lesions-USG, SCM-A2G remains competitive in
late-round accuracy and achieves the lowest volatility, indicating a more
stable global trajectory near convergence. Although FedAvg and MP-A2G can
attain comparable accuracy in some runs, their late-round behaviour is less
stable. On BAF, SCM-A2G provides the strongest result, achieving both the
highest late-round mean accuracy and the lowest volatility. Overall, these
results show that midpoint self-consistency improves the accuracy--stability
trade-off, especially in the BAF setting.
\begin{figure}[htb]
    \centering

    \begin{subfigure}[t]{0.24\textwidth}
        \centering
        \includegraphics[width=\linewidth]{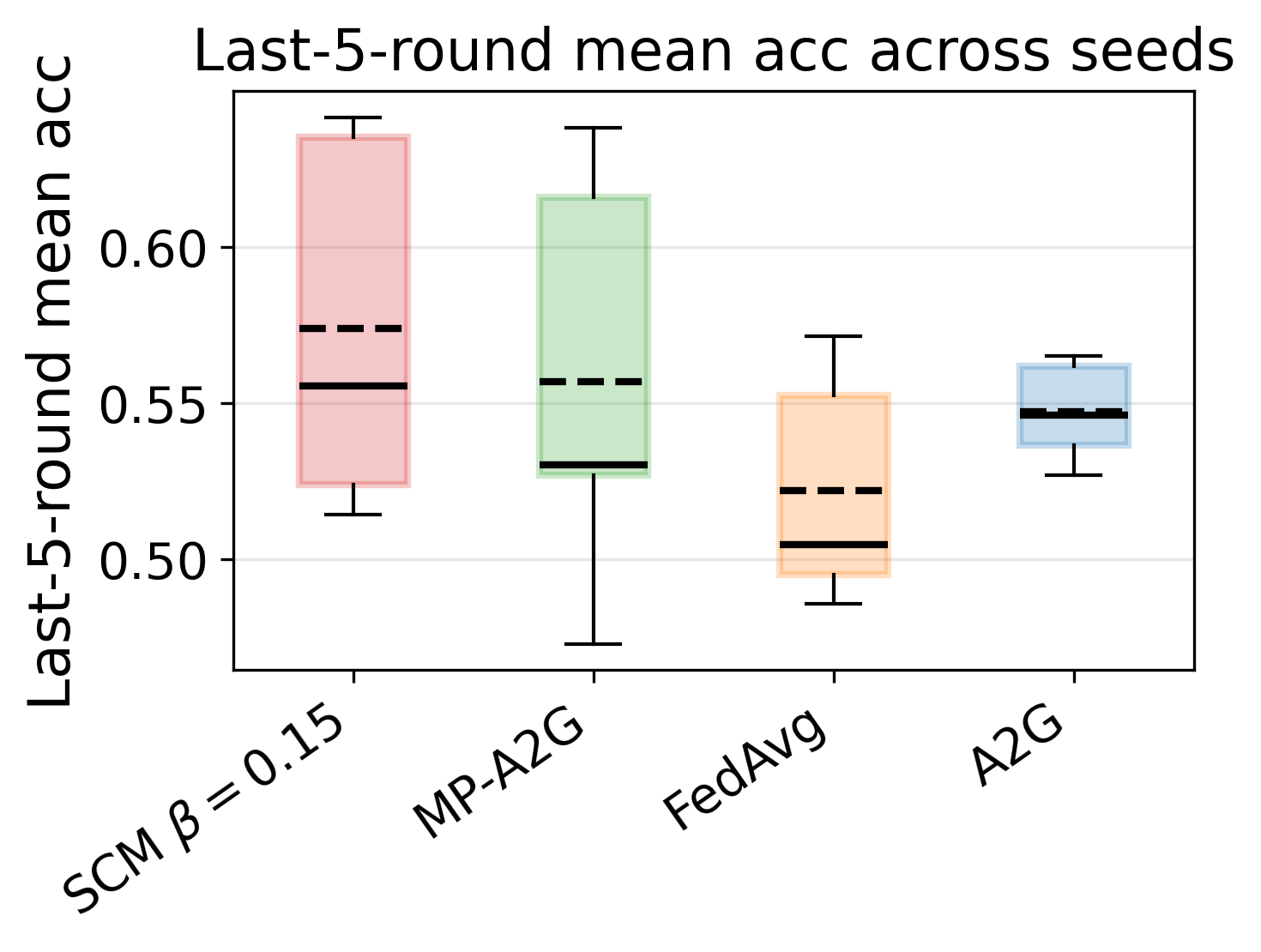}
        \caption{Breast-Lesions-USG}
        \label{fig:breast_last5_mean_accuracy}
    \end{subfigure}
    \hfill
    \begin{subfigure}[t]{0.24\textwidth}
        \centering
        \includegraphics[width=\linewidth]{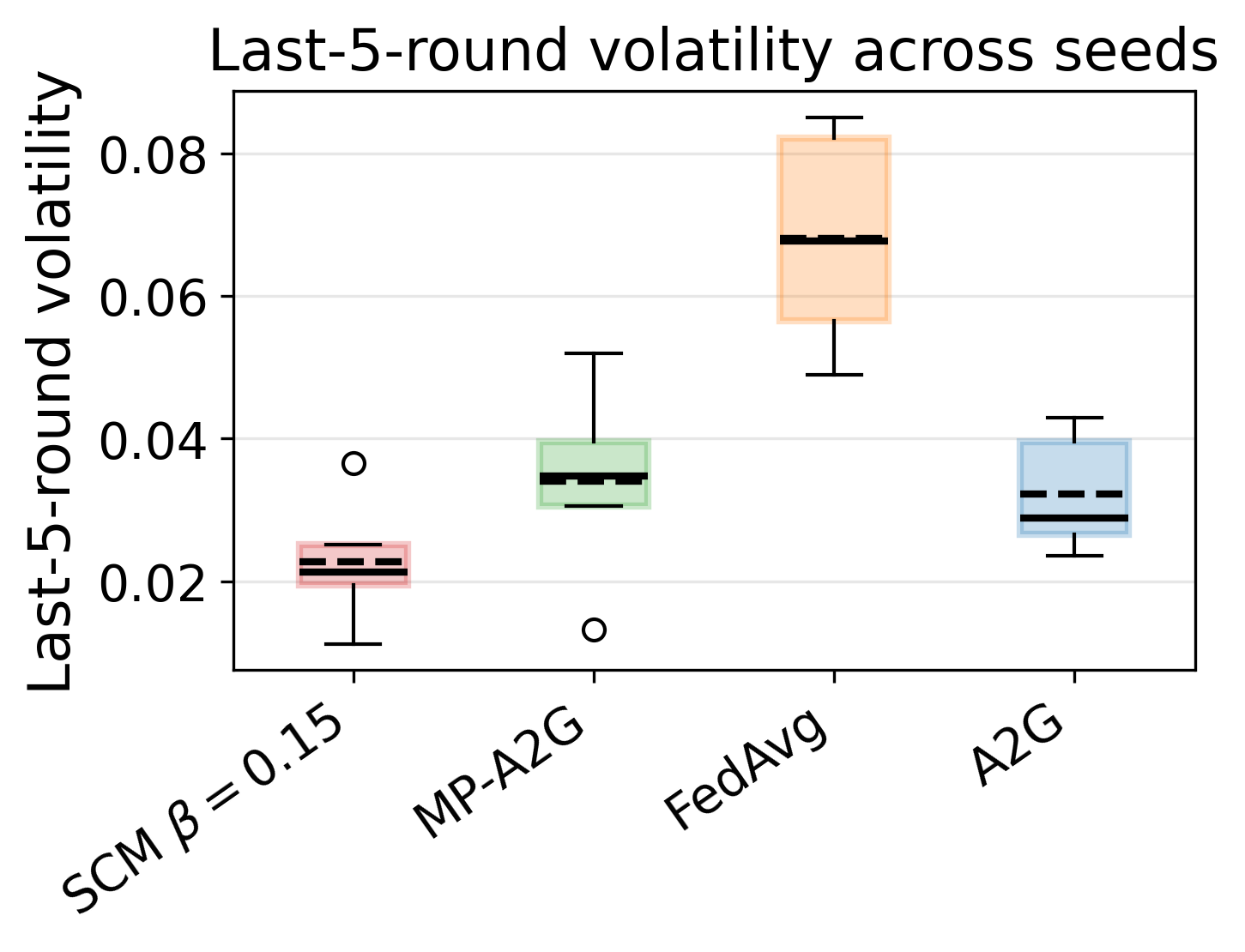}
        \caption{Breast-Lesions-USG}
        \label{fig:breast_last5_volatility}
    \end{subfigure}

    \hfill
    \begin{subfigure}[t]{0.24\textwidth}
        \centering
        \includegraphics[width=\linewidth]{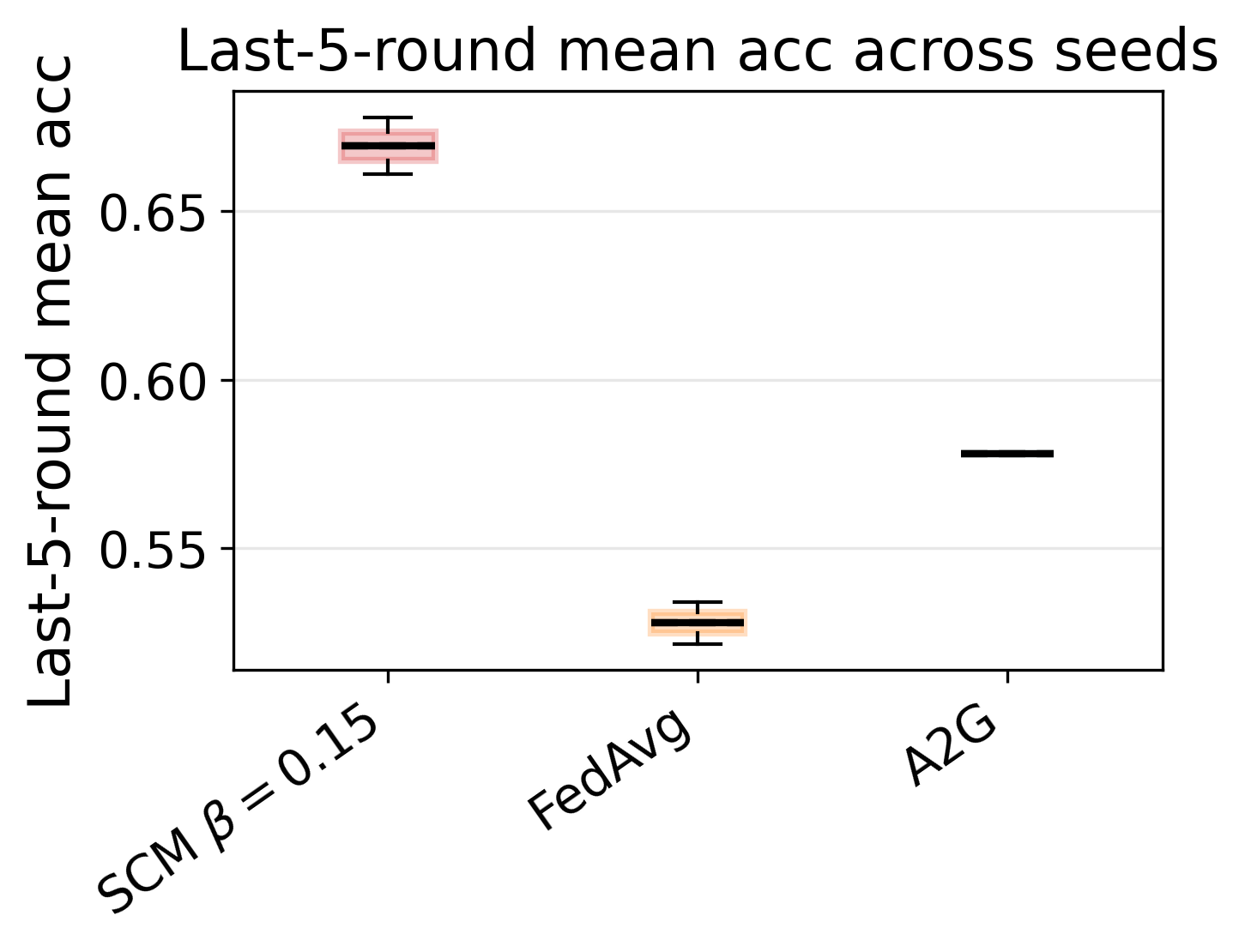}
        \caption{BAF}
        \label{fig:baf_last5_mean_accuracy}
    \end{subfigure}
    \hfill
    \begin{subfigure}[t]{0.24\textwidth}
        \centering
        \includegraphics[width=\linewidth]{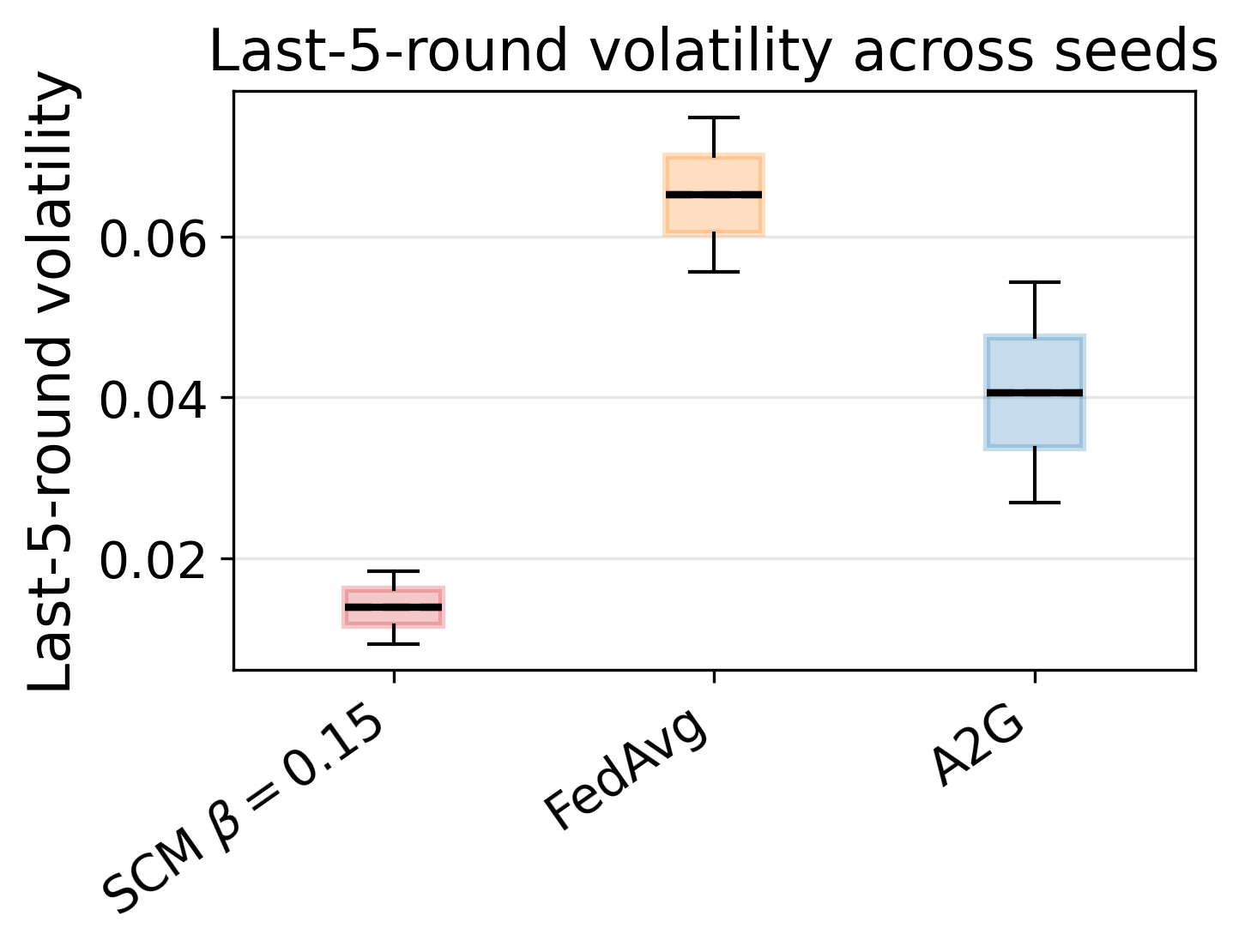}
        \caption{BAF}
        \label{fig:baf_breast_last5_volatility}
    \end{subfigure}

    \caption{Accuracy--stability behaviour on Breast-Lesions-USG and BAF. (a) and~(c) compare primary performance across seeds, while (b) and~(d) report last-5-round volatility. SCM-A2G achieves the
most favourable accuracy--stability trade-off: it remains competitive on
Breast-Lesions-USG and obtains stronger performance with lower volatility on
BAF.}
    \label{fig:accuracy_stability}
\end{figure}

\subsubsection{Performance and movement-control analysis}
\begin{figure}[htb]
    \centering
    \begin{subfigure}[t]{0.24\textwidth}
        \centering      \includegraphics[width=\linewidth]{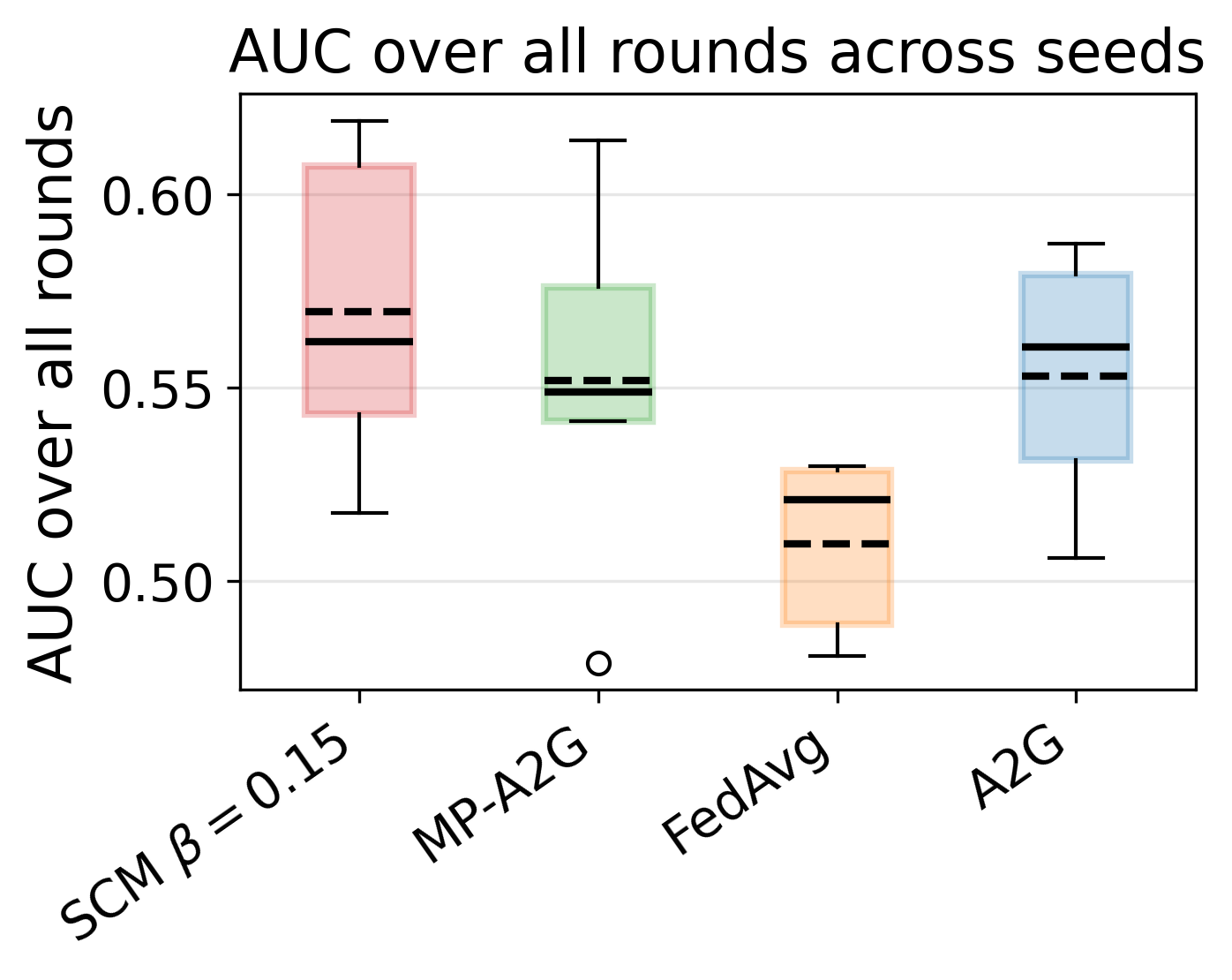}
        \caption{Breast-Lesions}
        \label{fig:breast_primary_perf_box}
    \end{subfigure}
    \hfill
    \begin{subfigure}[t]{0.24\textwidth}
        \centering
        \includegraphics[width=\linewidth]{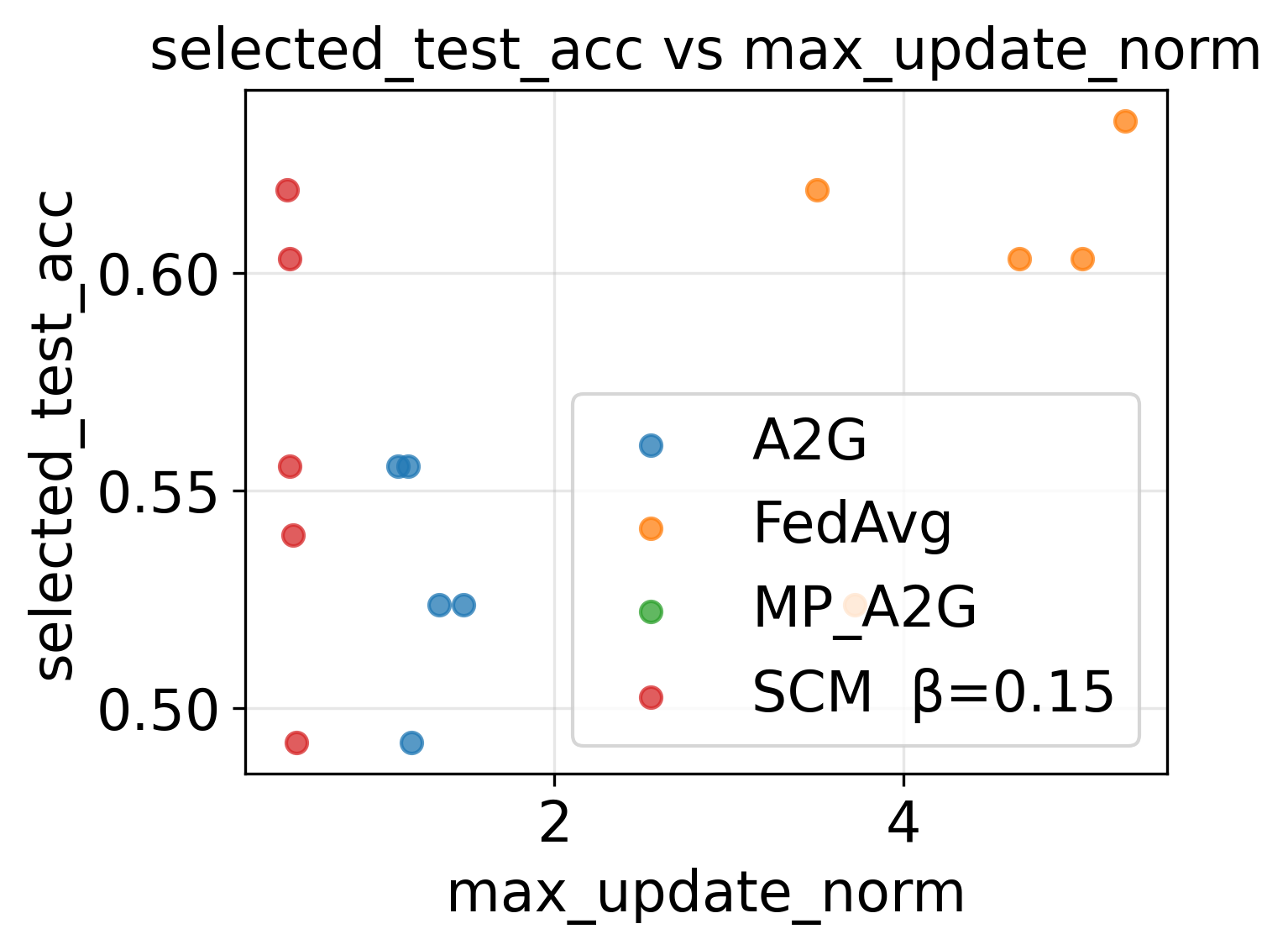}
        \caption{Breast-Lesions}
        \label{fig:breast_perf_update_scatter}
    \end{subfigure}

    \hfill\begin{subfigure}[t]{0.24\textwidth}
        \centering
        \includegraphics[width=\linewidth]{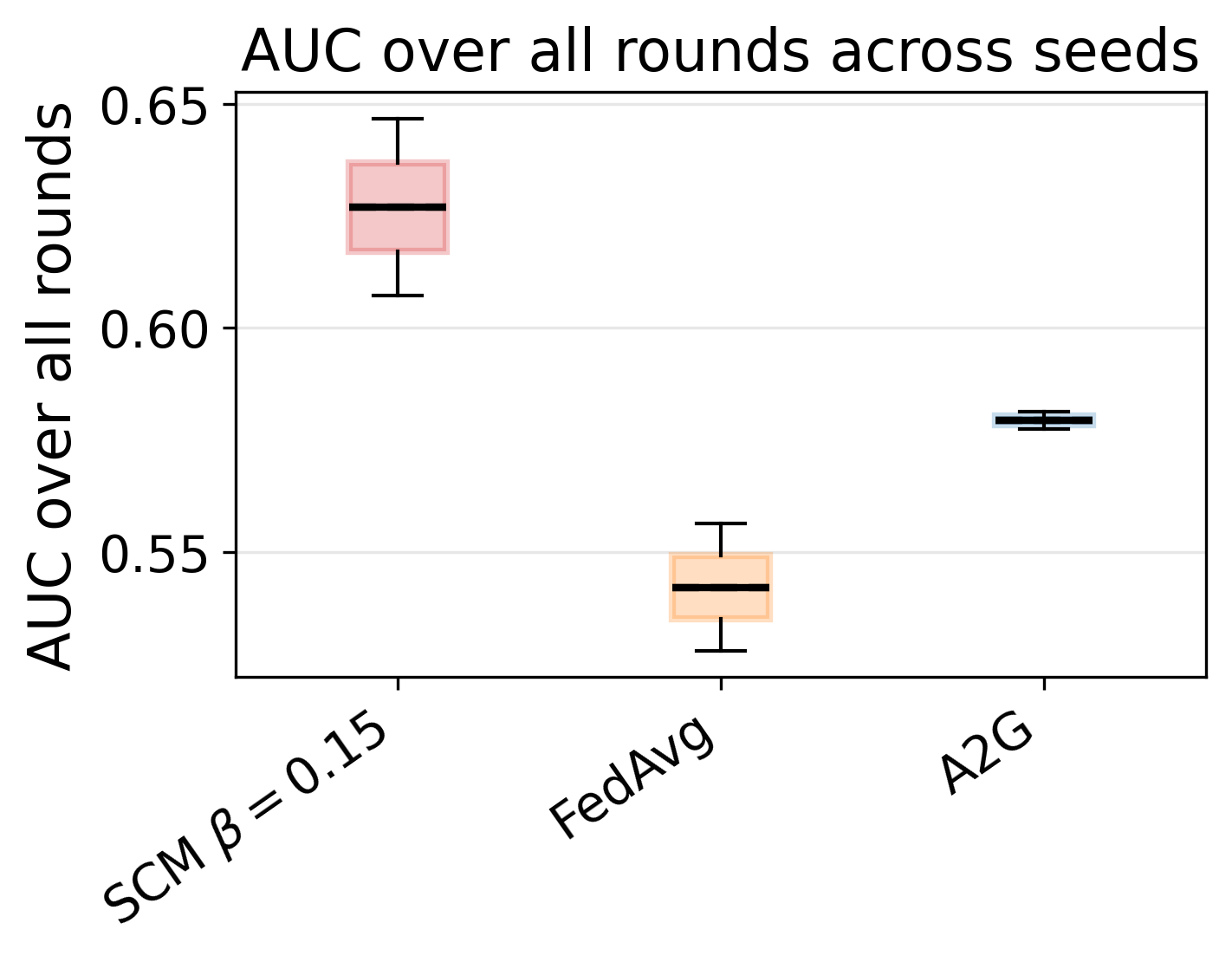}
        \caption{Baf Data AUC}
        \label{fig:baf_box_auc_all}
    \end{subfigure}
    \hfill
    \begin{subfigure}[t]{0.24\textwidth}
        \centering
        \includegraphics[width=\linewidth]{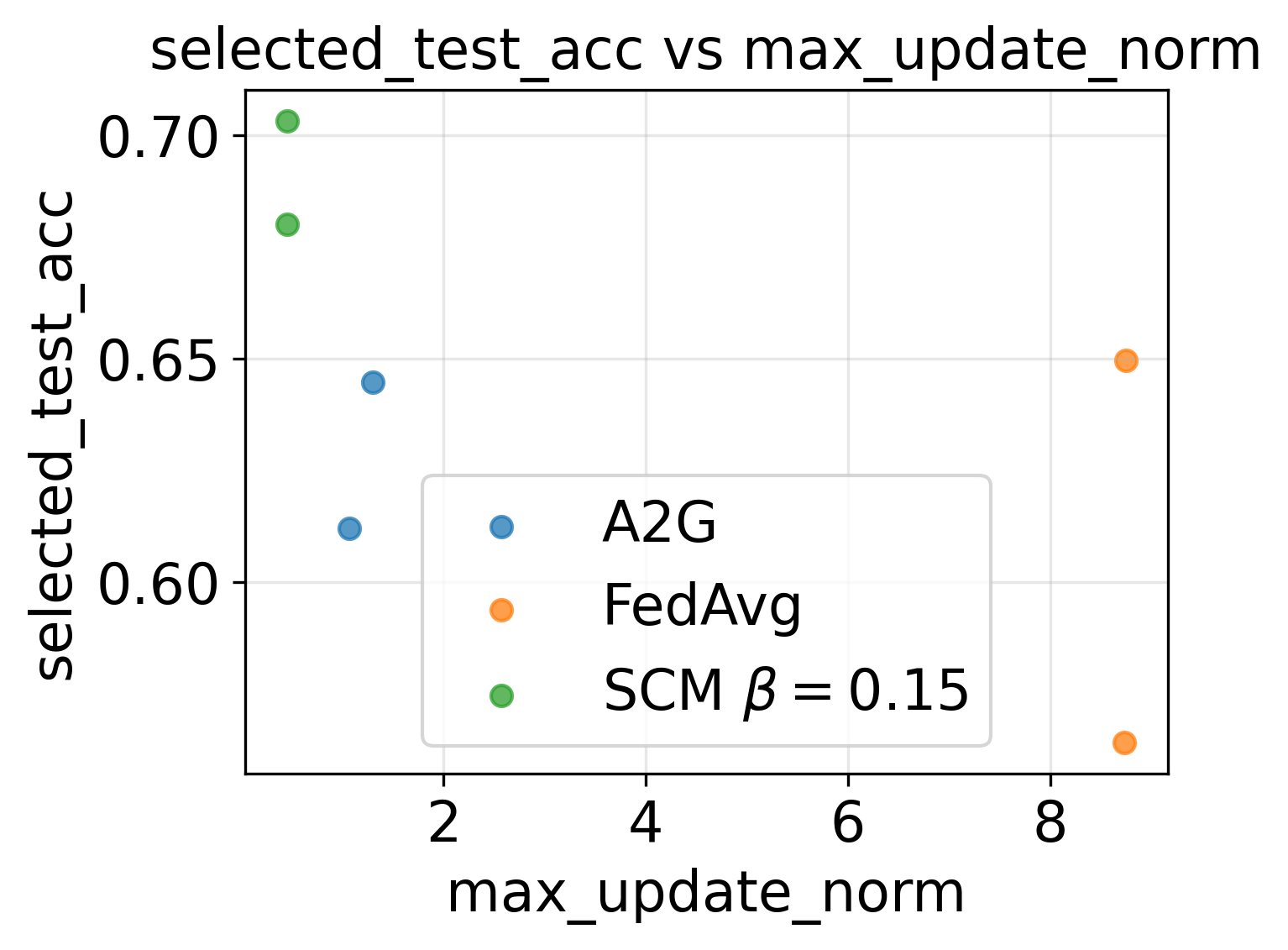}
        \caption{Baf update norm.}
        \label{fig:baf_perf_update_scatter}
    \end{subfigure}
    \caption{Performance and movement-control analysis on Breast-Lesions-USG and
BAF. Subfigures~(a) and~(c) report AUC over all federated rounds, while
Subfigures~(b) and~(d) relate selected test accuracy to the maximum server
update norm. SCM-A2G achieves competitive or superior trajectory-level
performance while requiring substantially smaller server movements than
FedAvg, supporting the implicit damping effect of self-consistent midpoint
aggregation.}
    \label{fig:performance_movement}
\end{figure}

Figure~\ref{fig:performance_movement} shows how predictive behavior relates to movement.
control. On the Breast-Lesions-USG dataset, SCM-A2G achieves competitive trajectory-level
AUC while keeping the maximum server update norm much lower than
FedAvg. FedAvg sometimes reaches high selected accuracy, but this is
linked to aggressive update magnitudes, which suggests a less controlled
aggregation path. A similar pattern appears on the BAF dataset, where SCM-A2G achieves
the highest AUC while keeping the maximum update norm much smaller than
FedAvg. These results support the main idea behind SCM-A2G: the method does
not just optimize for endpoint accuracy, but also controls the accepted server
movement by using midpoint self-consistency.
%Figure 07

\subsubsection{Predictive and validation behaviour}

\begin{figure}[htb]
    \centering

    \begin{subfigure}[t]{0.24\textwidth}
        \centering
        \includegraphics[width=\linewidth]{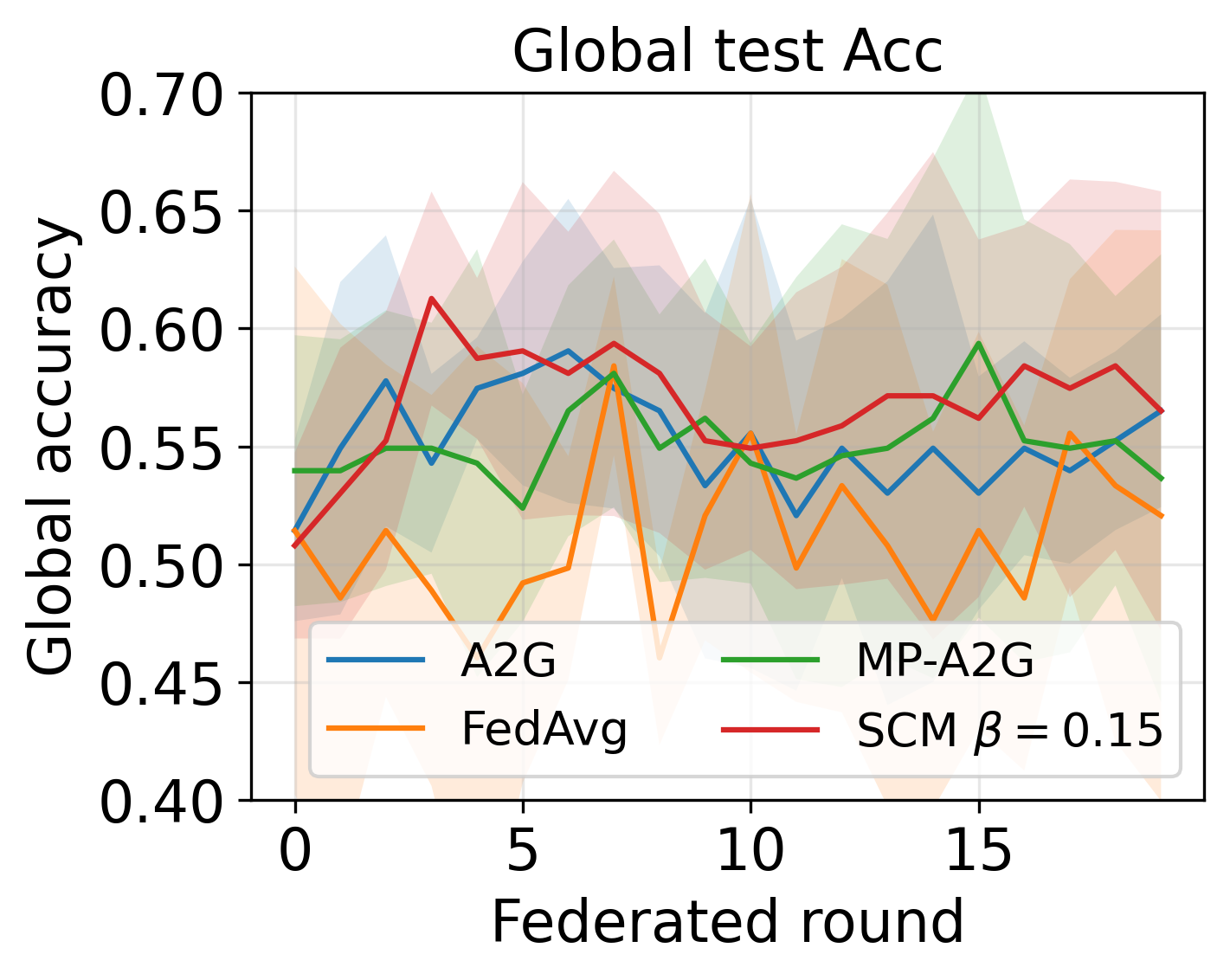}
        \caption{Breast-Lesion}
        \label{fig:breast_global_accuracy}
    \end{subfigure}
    \hfill
    \begin{subfigure}[t]{0.24\textwidth}
        \centering
        \includegraphics[width=\linewidth]{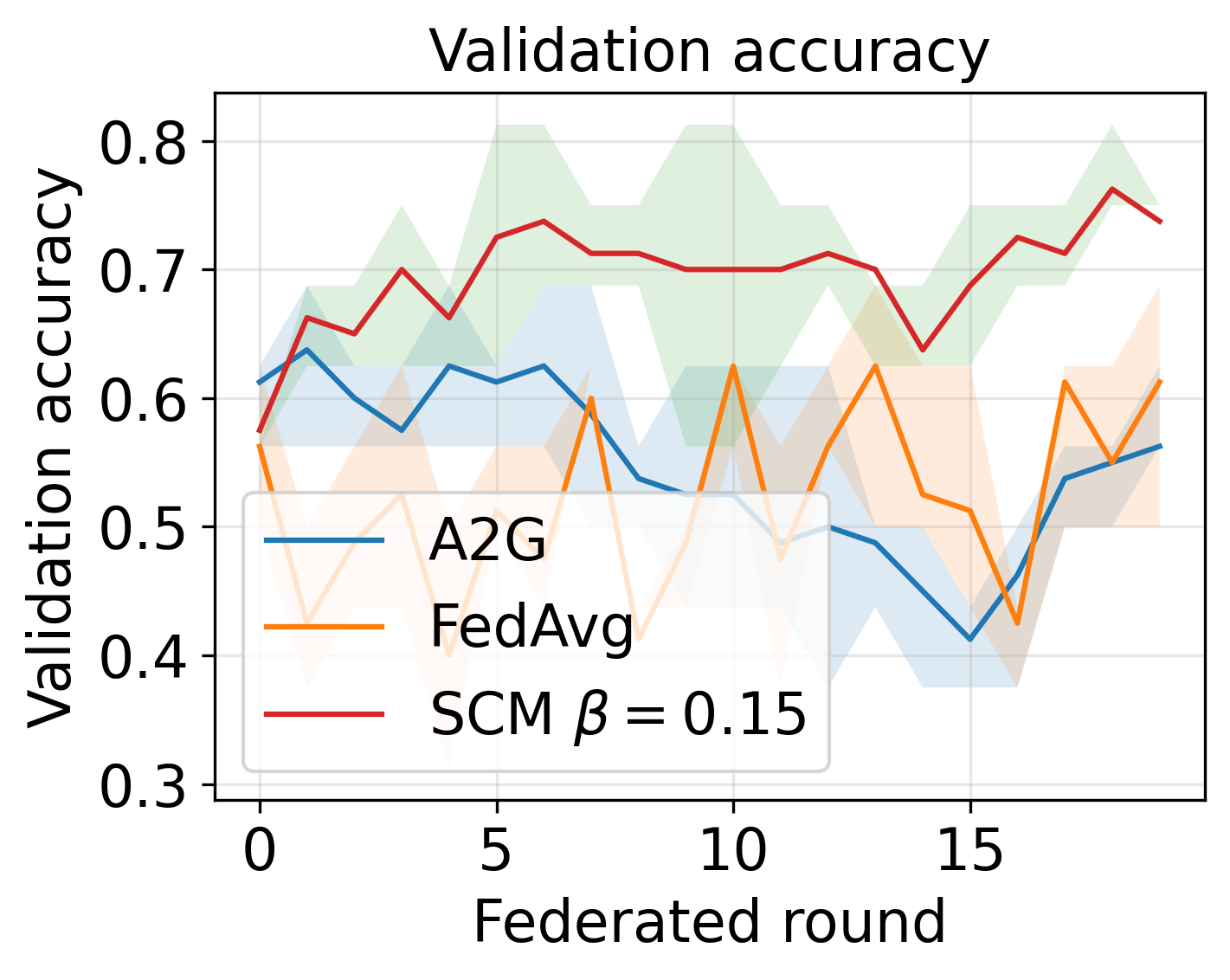}
        \caption{Breast-Lesion Val Acc}
        \label{fig:breast_validation_accuracy}
    \end{subfigure}

\hfill

    \begin{subfigure}[t]{0.24\textwidth}
        \centering
        \includegraphics[width=\linewidth]{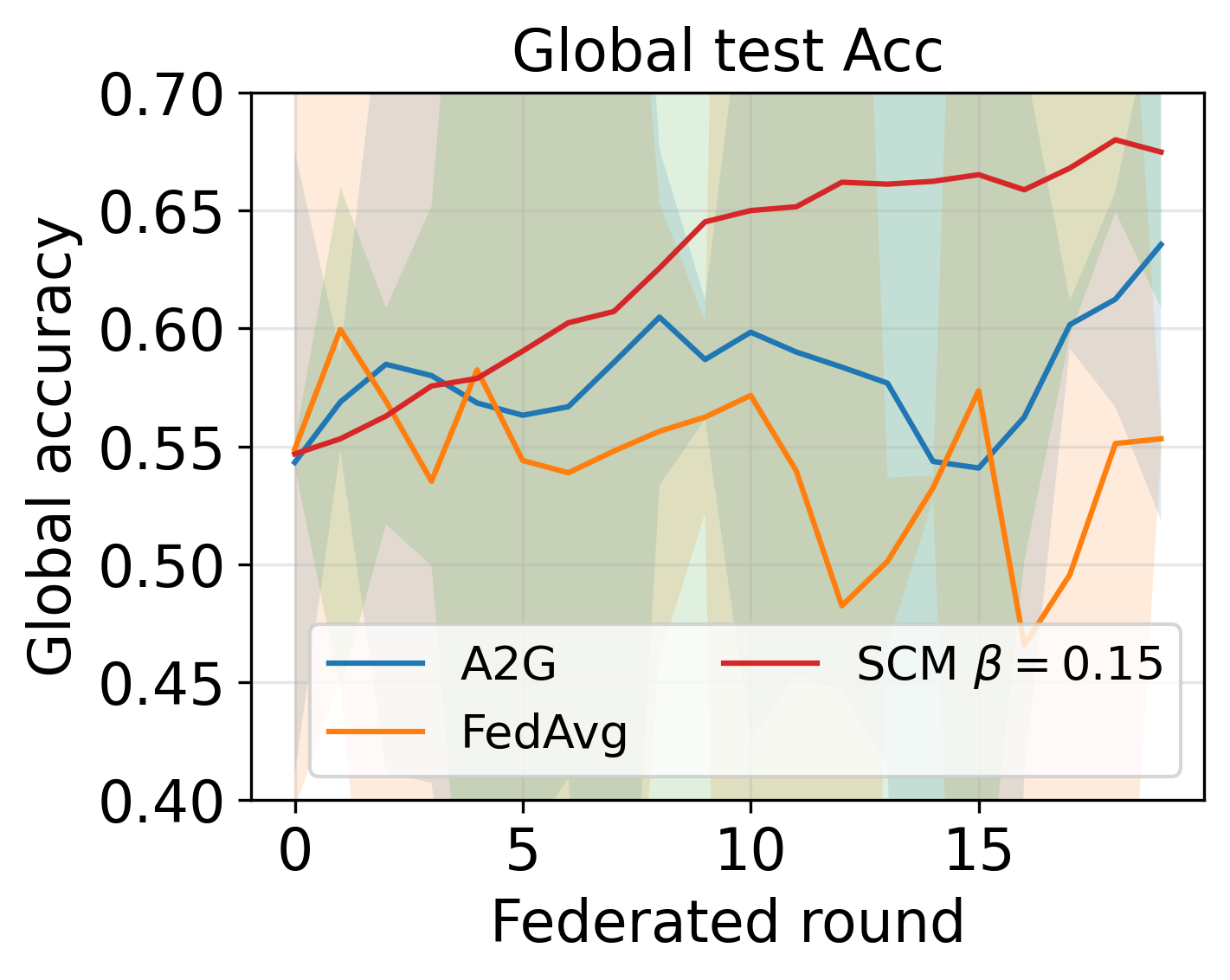}
        \caption{BAF}
        \label{fig:baf_plot_accuracy_mean_ci}
    \end{subfigure}
    \hfill
    \begin{subfigure}[t]{0.24\textwidth}
        \centering
        \includegraphics[width=\linewidth]{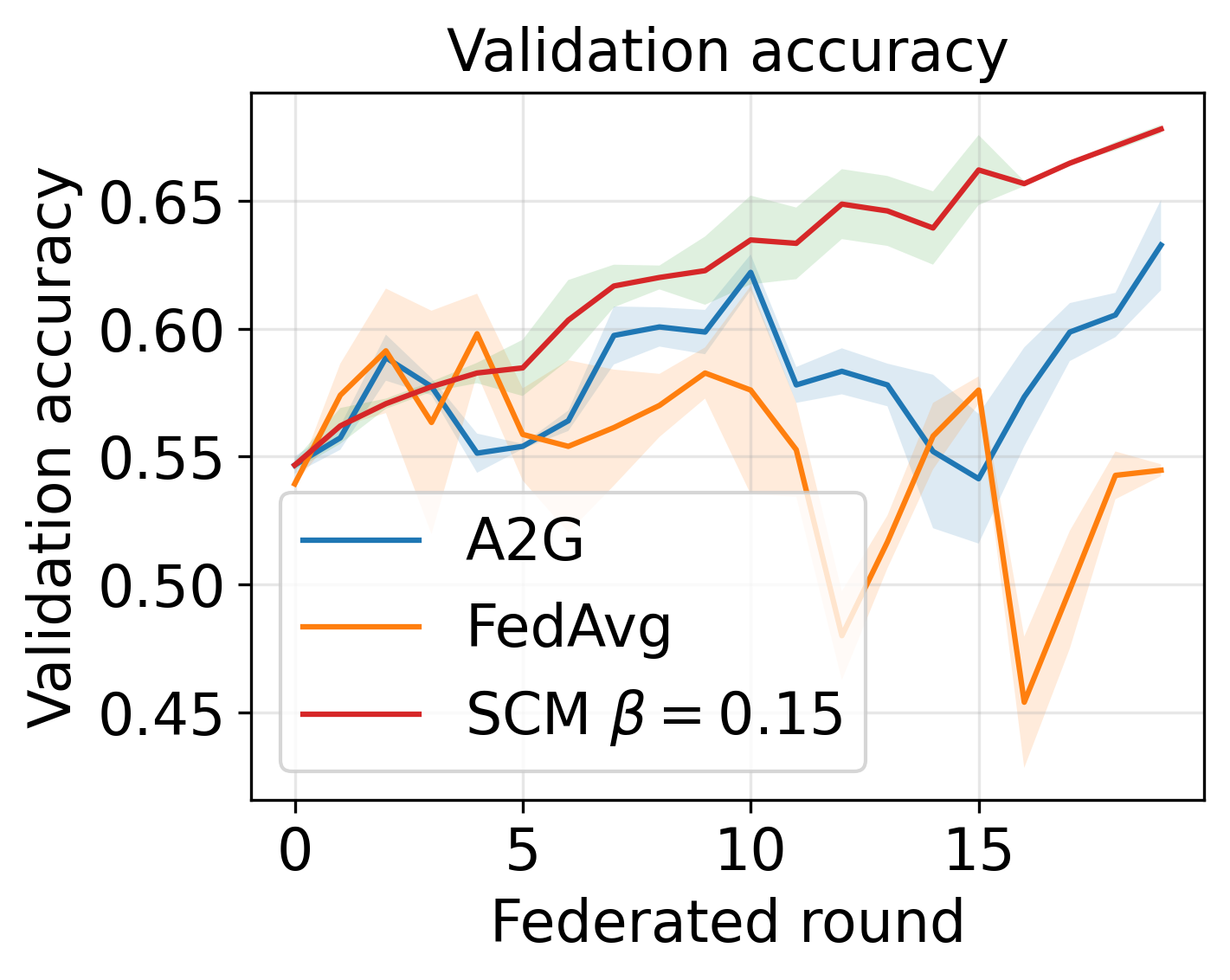}
        \caption{BAF Val Acc}
        \label{fig:baf_plot_validation_accuracy}
    \end{subfigure}
    \caption{Predictive and validation behaviour on Breast-Lesions-USG and BAF.
Subfigures~(a) and~(b) correspond to Breast-Lesions-USG, while
Subfigures~(c) and~(d) correspond to BAF. SCM-A2G maintains competitive global
test accuracy and shows stronger validation behaviour, indicating improved
generalization under heterogeneous federated settings.}
    \label{fig:predictive_validation}
\end{figure}

\begin{figure}[htb]
    \centering

    \begin{subfigure}[t]{0.24\textwidth}
        \centering
        \includegraphics[width=\linewidth]{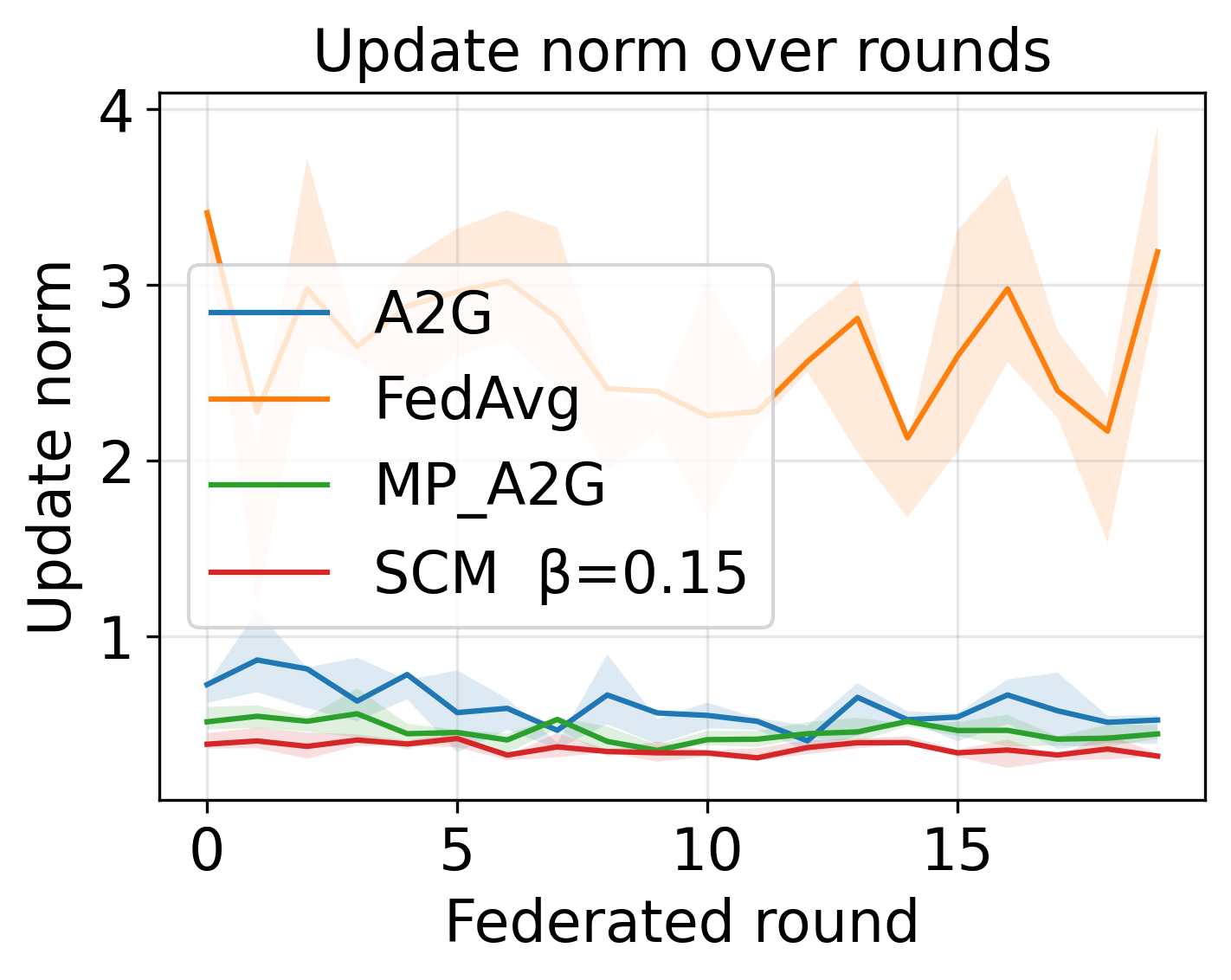}
        \caption{Server update norm.}
        \label{fig:breast_update_norm}
    \end{subfigure}
    \hfill
    \begin{subfigure}[t]{0.24\textwidth}
        \centering
        \includegraphics[width=\linewidth]{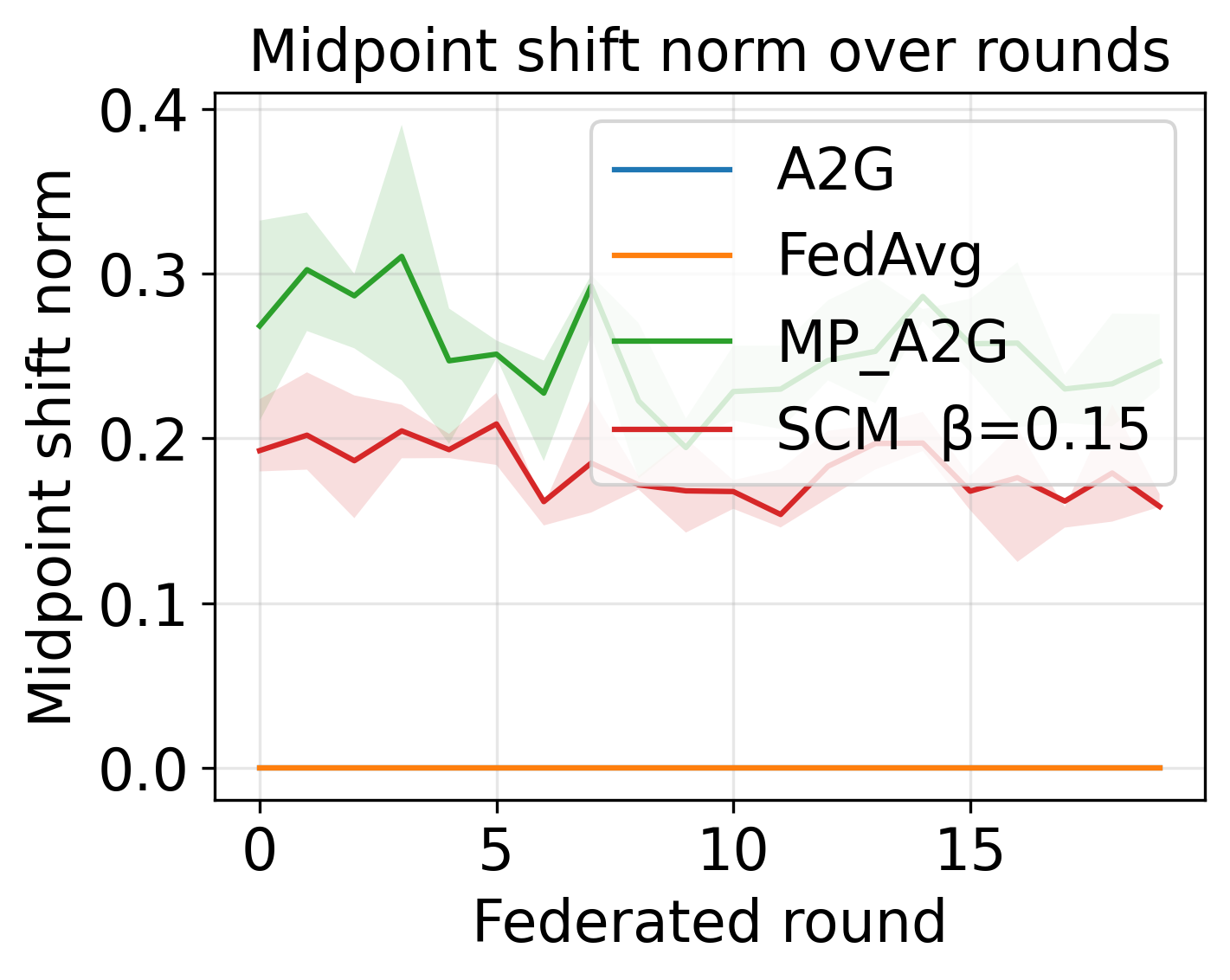}
        \caption{Midpoint shift norm.}
        \label{fig:breast_midpoint_shift}
    \end{subfigure}
    \hfill\begin{subfigure}[t]{0.24\textwidth}
        \centering
        \includegraphics[width=\linewidth]{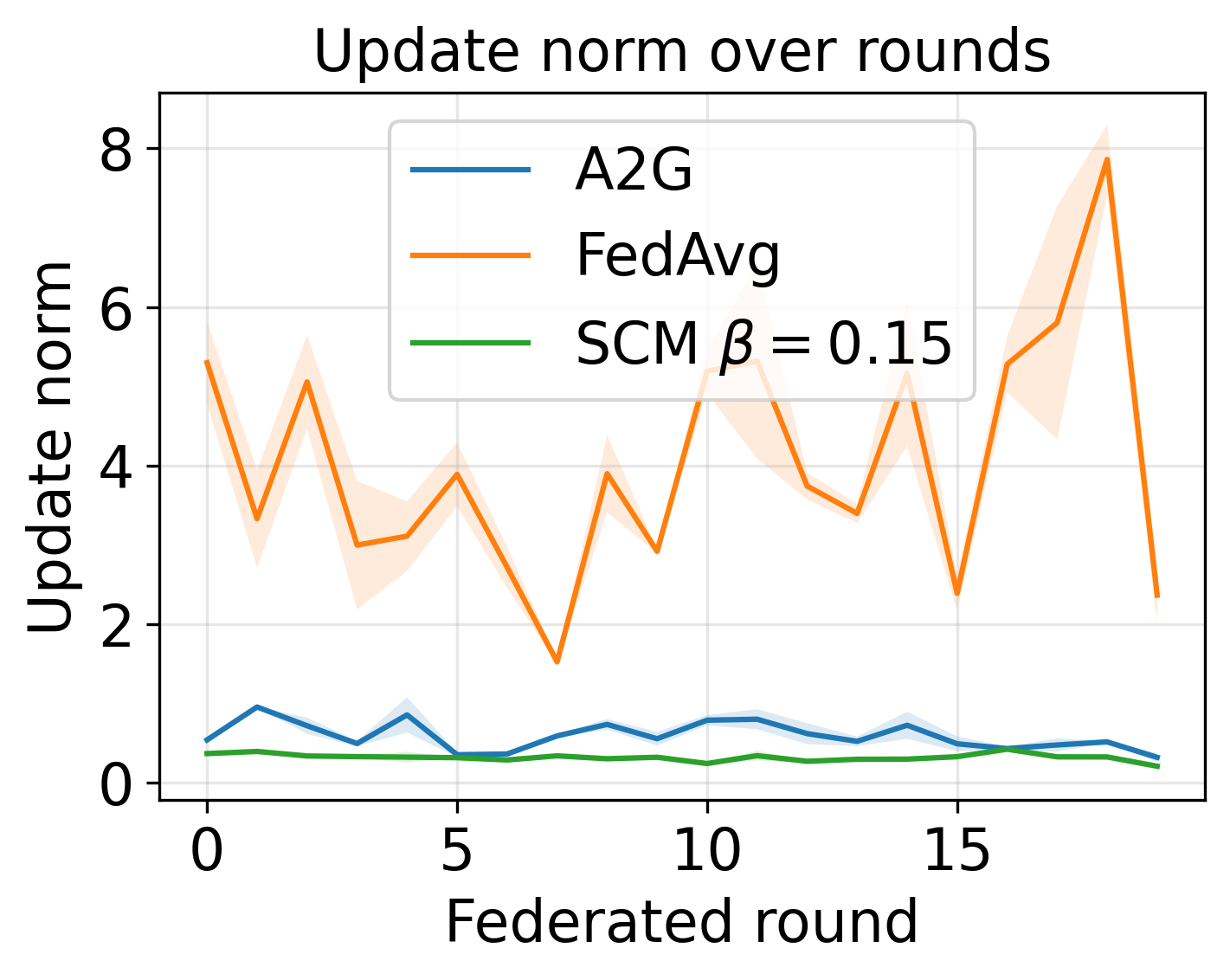}
        \caption{Server update norm.}
        \label{fig:baf_plot_diag_update_norm}
    \end{subfigure}
    \hfill
    \begin{subfigure}[t]{0.24\textwidth}
        \centering
        \includegraphics[width=\linewidth]{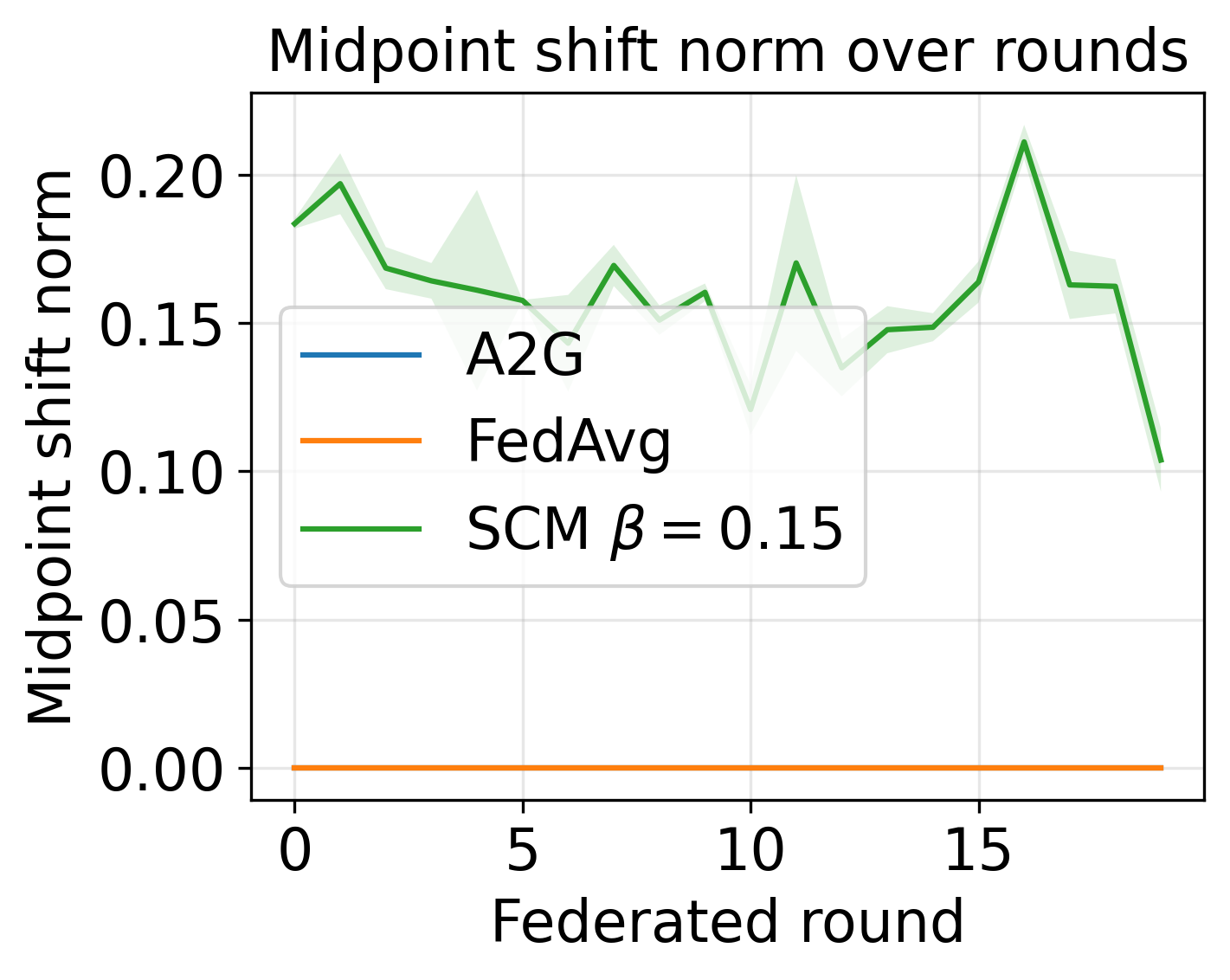}
        \caption{Midpoint shift norm.}
   \label{fig:baf_plot_diag_midpoint_shift_norm}
    \end{subfigure}

    \caption{Aggregation-path diagnostics on the Breast-Lesions-USG dataset and BAF dataset.
    }
    \label{fig:aggregation_diagnostics}
\end{figure}
Figure~\ref{fig:predictive_validation} shows the test and validation accuracy for each round.
The test curves for Breast-Lesions-USG are quite noisy, which reflects the
stochastic nature of quantum neural network (QNN) training and the presence of non-independent and identically distributed (non-IID) clinical partitions. Nevertheless,
the validation curve demonstrates that SCM-A2G maintains stronger generalization
behaviour over most rounds. For the BAF dataset, the trend is more pronounced: SCM-A2G improves
steadily and achieves the highest validation accuracy by the final rounds.
Collectively, these results indicate that midpoint self-consistency does not merely
reduce update magnitudes; it also contributes to the development of global models with improved
validation performance across domains.

\subsubsection{Aggregation-path diagnostics}

Fig.~\ref{fig:aggregation_diagnostics} explains the mechanism behind the
observed stability. In both Breast-Lesions-USG and BAF, FedAvg produces the
largest and most volatile server update norms, indicating aggressive movement
of the global model. A2G and MP-A2G reduce this effect, but SCM-A2G produces the
most controlled update trajectory. The midpoint-shift plots further distinguish
SCM-A2G from MP-A2G: while MP-A2G applies a one-shot midpoint correction,
SCM-A2G iteratively recomputes midpoint-supported directions and accepts the
movement only after self-consistency is reached. The lower midpoint-shift norm
therefore provides direct empirical evidence that SCM-A2G stabilizes the
accepted server movement.

Overall, the results show that SCM-A2G is not simply an accuracy-improving
heuristic. Its main advantage is the accuracy--stability--movement-control
trade-off. FedAvg can occasionally achieve high accuracy, but it does so with
large and unstable update movements. A2G improves geometry awareness, and
MP-A2G introduces a one-shot midpoint correction, but SCM-A2G further stabilizes
the server trajectory by requiring the accepted movement to be supported by its
own midpoint. This leads to lower late-round volatility, smaller update norms,
and stronger validation behaviour, especially on the BAF dataset. The
consistent behaviour across Breast-Lesions-USG and BAF supports the claim that
SCM-A2G is domain-general rather than tuned to a single application.

\begin{figure}[h]
    \centering
    \begin{subfigure}[t]{0.24\textwidth}
        \centering
        \includegraphics[width=\linewidth]{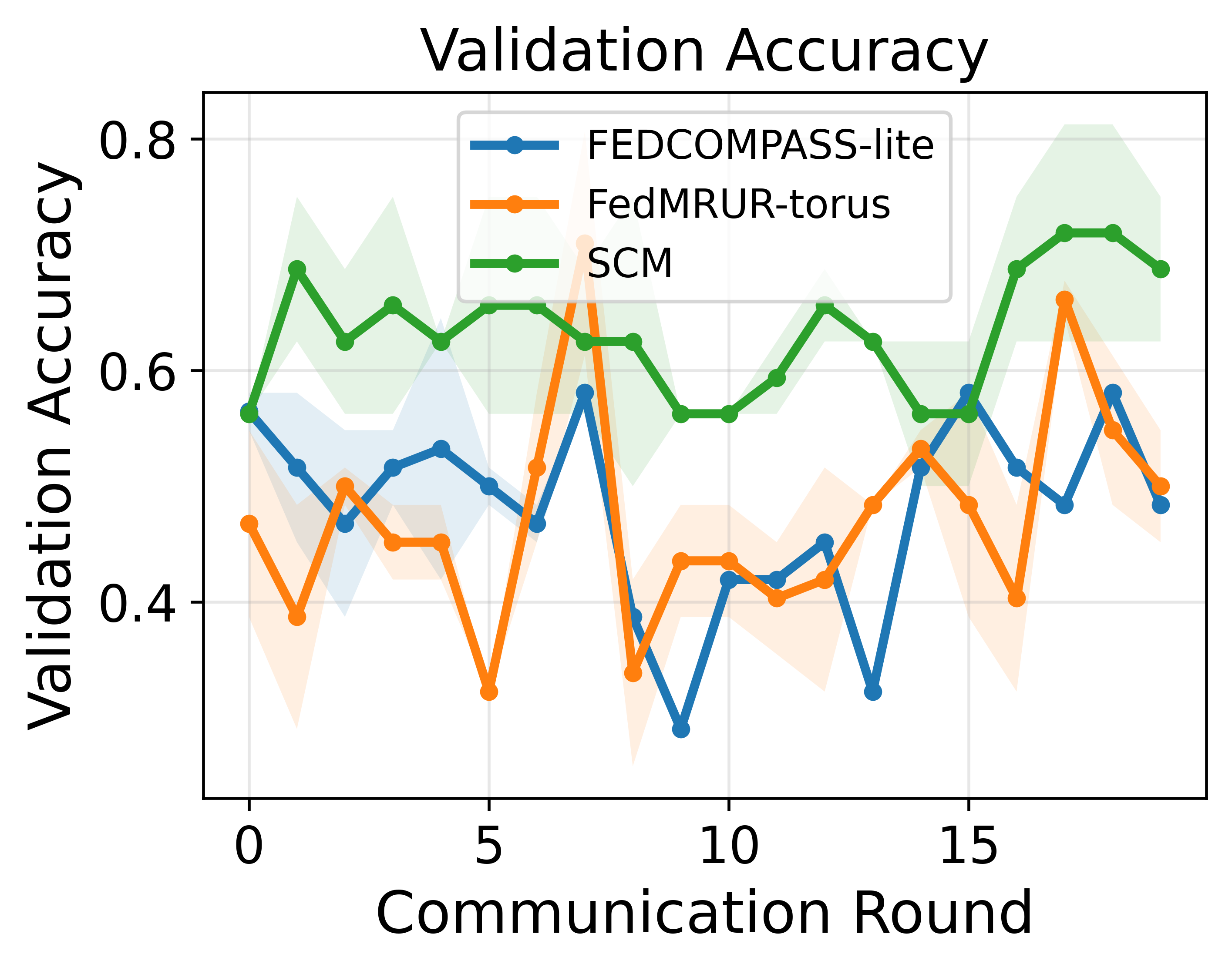}
        \caption{}
        \label{fig:comparativeVal_acc}
    \end{subfigure}
    \hfill
    \begin{subfigure}[t]{0.24\textwidth}
        \centering
        \includegraphics[width=\linewidth]{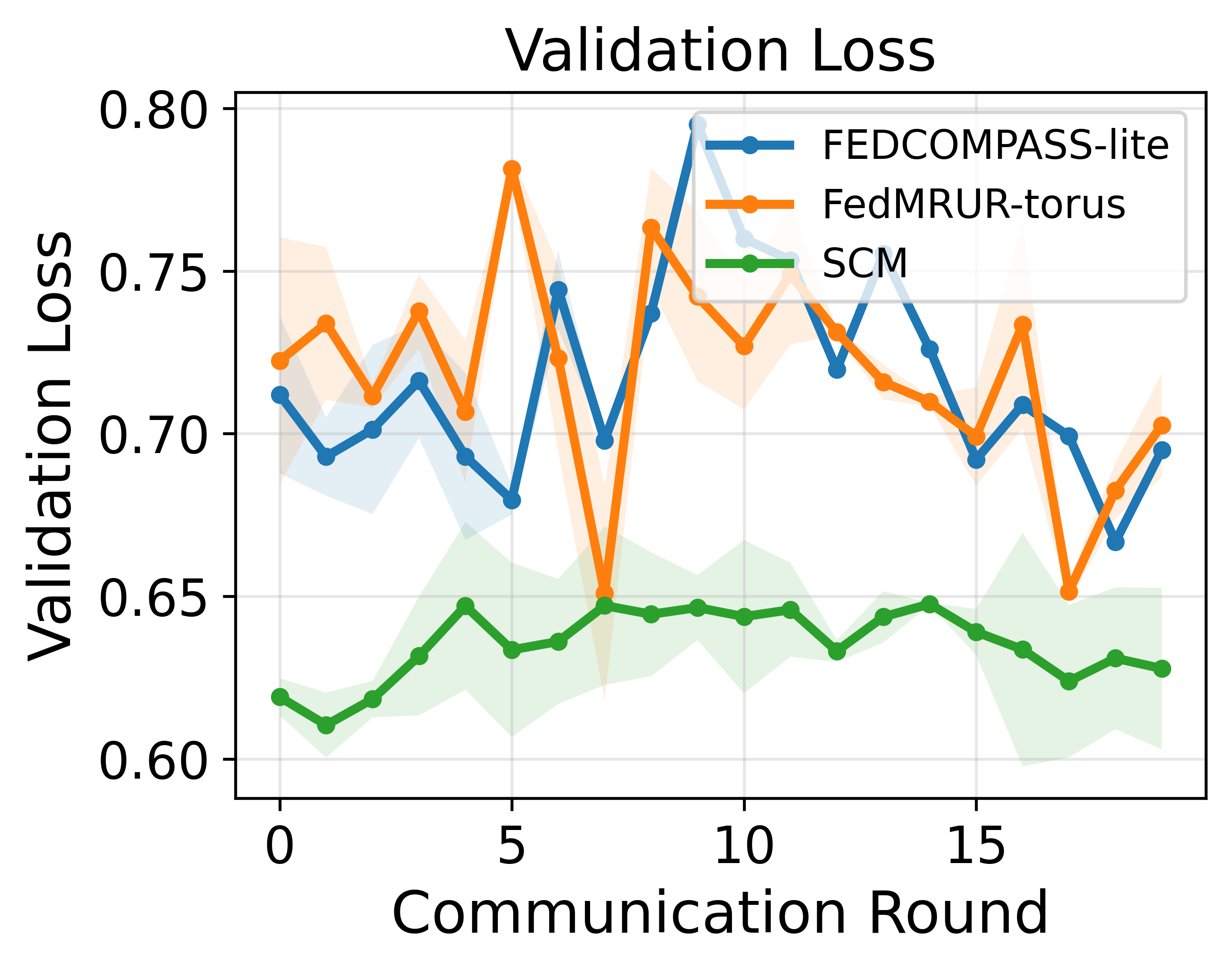}
        \caption{}
        \label{fig:comparative_val_loss}
    \end{subfigure}
    \caption{Comparative validation performance of SCM-A2G against geometry-aware
baselines on the Breast-Lesions-USG dataset. (a) reports validation
accuracy across communication rounds, while (b) reports validation
loss.}
\label{fig:comparative_validation_results}
    \label{fig:comparative_basleine}
\end{figure}

Furthermore, Fig.~\ref{fig:comparative_validation_results} compares SCM-A2G with two
geometry-aware baselines, FEDCOMPASS \cite{Wang2026FEDCOMPASS} and FedMRUR \cite{An2023FedMRUR}. 
Overall, SCM-A2G outperforms the geometry-aware baselines in both validation
accuracy and validation loss, supporting the claim that midpoint
self-consistency provides an additional stabilization benefit beyond
geometry-aware aggregation alone.

\section{Conclusion}

This paper introduced SCM-A2G-QFL, a self-consistent midpoint aggregation
framework for QoS- and geometry-aware quantum federated learning. The proposed
method addresses two coupled challenges in QFL: heterogeneous client reliability
and periodic QNN parameter geometry. Unlike conventional Euclidean aggregation,
SCM-A2G-QFL computes QoS-aware client weights, constructs a torus-consistent
candidate direction, and accepts the next global movement only when it remains
supported by its own midpoint. Theoretical analysis showed that the resulting
update is normalized, geometry-preserving, locally stable under small geometry
gains, and implicitly damped relative to direct A2G movement. IBM hardware
validation further demonstrated that aggregation geometry can induce physically
different quantum observables, confirming the practical relevance of
geometry-aware QNN aggregation. Overall, SCM-A2G-QFL provides a principled
server-side aggregation mechanism for stabilizing quantum federated learning
under noisy, heterogeneous, and geometry-sensitive conditions.

\bibliographystyle{IEEEtran}
\bibliography{ref}
\appendices

\appendix

\subsection{IBM Hardware Validation}
\label{subsec:ibm_hardware_validation}

\begin{table*}[htb]
\centering
\scriptsize
\caption{IBM Quantum cloud real validation of controlled angular aggregation outputs on \texttt{ibm\_fez}. 
The output angle \(\theta_{\mathrm{out}}\) is encoded as a single-qubit \(R_y(\theta)\) circuit and evaluated using the Pauli-\(Z\) expectation. 
\(\Delta\theta_t\) denotes the accepted angular movement from the current server parameter \(\theta_t=-170^\circ\).}
\label{tab:ibm_controlled_validation}
\begin{tabular}{c l r r r r r r}
\toprule
\textbf{Case} &
\textbf{Method} &
\(\boldsymbol{\theta_{\mathrm{out}}}\) &
\(\boldsymbol{\Delta\theta_t}\) &
\textbf{Depth} &
\(\boldsymbol{Z_{\mathrm{sim}}}\) &
\(\boldsymbol{Z_{\mathrm{IBM}}}\) &
\(\boldsymbol{| \Delta Z |}\) \\
\midrule

\multirow{4}{*}{A}
& Euclidean mean & -0.97 & 169.03 & 4 & 0.9999 & 0.9997 & 0.0002 \\
& Circular mean  & -0.95 & 169.05 & 4 & 0.9999 & 0.9987 & 0.0012 \\
& MP-A2G         & -146.55 & 23.45 & 4 & -0.8343 & -0.8457 & 0.0114 \\
& SCM-A2G        & -146.41 & 23.59 & 4 & -0.8331 & -0.8401 & 0.0071 \\

\midrule

\multirow{4}{*}{B}
& Euclidean mean & 0.00 & 170.00 & 0 & 1.0000 & 0.9977 & 0.0023 \\
& Circular mean  & -180.00 & -10.00 & 2 & -1.0000 & -0.9835 & 0.0165 \\
& MP-A2G         & -171.39 & -1.39 & 4 & -0.9887 & -0.9951 & 0.0064 \\
& SCM-A2G        & -171.40 & -1.40 & 4 & -0.9887 & -0.9977 & 0.0089 \\

\midrule

\multirow{4}{*}{C}
& Euclidean mean & 0.00 & 170.00 & 0 & 1.0000 & 1.0047 & 0.0047 \\
& Circular mean  & 0.00 & 170.00 & 0 & 1.0000 & 1.0032 & 0.0032 \\
& MP-A2G         & -163.06 & 6.94 & 4 & -0.9566 & -0.9659 & 0.0092 \\
& SCM-A2G        & -163.02 & 6.98 & 4 & -0.9564 & -0.9699 & 0.0135 \\

\midrule

\multirow{4}{*}{D}
& Euclidean mean & 58.07 & 228.07 & 4 & 0.5289 & 0.4963 & 0.0326 \\
& Circular mean  & 178.06 & -11.94 & 4 & -0.9994 & -0.9951 & 0.0043 \\
& MP-A2G         & -171.66 & -1.66 & 4 & -0.9894 & -0.9992 & 0.0098 \\
& SCM-A2G        & -171.67 & -1.67 & 4 & -0.9894 & -0.9946 & 0.0052 \\

\bottomrule
\end{tabular}
\vspace{1mm}
\begin{flushleft}
\footnotesize
All angles are reported in degrees. \(Z_{\mathrm{sim}}\) denotes the ideal simulator expectation and \(Z_{\mathrm{IBM}}\) denotes the IBM hardware expectation. 
\(|\Delta Z|=|Z_{\mathrm{IBM}}-Z_{\mathrm{sim}}|\). 
The SCM residual is not shown here because it is an algorithmic fixed-point quantity rather than a hardware-measured observable. 
All transpiled circuits used zero CNOT gates on the selected backend layout.
\end{flushleft}
\end{table*}

The hardware results show that Euclidean and circular/SCM-based aggregation can
produce substantially different quantum observables. For wrap-around cases,
Euclidean aggregation may yield angles whose ideal observables are positive,
whereas circular, MP-A2G, and SCM-A2G remain near the seam-consistent region and
produce negative observables close to the simulator prediction. This confirms
that the aggregation geometry affects the quantum STATE realised on hardware.

Importantly, the IBM experiment is not intended to claim that full federated QNN
training was executed on hardware. Rather, it provides hardware-backed
validation that the angles generated by SCM-A2G are executable on a real backend
and preserve the expected quantum observable behaviour.
All experiments utilized a single NVIDIA Tesla T4 graphics processing unit (GPU) and CUDA version 12.4 in a high-RAM runtime environment.

Table~\ref{tab:ibm_controlled_validation} validates that aggregation-induced angles produce the expected observable behavior on real IBM hardware. In the wrap-around case B, Euclidean averaging gives \(\theta=0^\circ\), producing a positive hardware expectation \(Z_{\mathrm{IBM}}=0.9977\). In contrast, circular, MP-A2G, and SCM-A2G remain near the seam-consistent region and produce negative expectations close to \(-1\). Similarly, in case D, the Euclidean aggregate gives \(Z_{\mathrm{IBM}}=0.4963\), whereas circular and SCM-based updates produce negative values close to simulator predictions. These results show that aggregation geometry has a direct physical effect on the quantum observable.

\section{Proof of Theorem~\ref{theorem:main_convergence_result}}
\label{app:proof_scm_convergence}

% ============================================================
% This appendix proves the main convergence theorem.
% We use Option A: the SCM residual enters the one-step descent
% bound as C_3 R_{\mathrm{SCM},t}^2.
% Therefore, the final convergence bound contains
% (C_3/(\beta T)) \sum_t R_{\mathrm{SCM},t}^2,
% and the simplified residual contribution becomes
% O(\varepsilon_{\mathrm{SCM}}^2/\beta).
% ============================================================

\subsection{Proof of Lemma~\ref{lemma:scm_one_step_descent}}

% ------------------------------------------------------------
% Step 1: Define the total perturbation term.
% This combines all non-ideal effects: local stochasticity,
% client heterogeneity, QoS weighting error, and midpoint error.
% ------------------------------------------------------------

Let
\[
\boldsymbol{\Xi}_t
=
\boldsymbol{\xi}_{l,t}
+
\boldsymbol{\xi}_{g,t}
+
\boldsymbol{\xi}_{q,t}
+
\boldsymbol{\xi}_{m,t}
\]
denote the total perturbation induced by stochastic local optimization,
client heterogeneity, QoS-dependent weighting, and midpoint approximation.

In the local tangent chart around \(\boldsymbol{\theta}_t\), let
\(\mathbf{u}_t^\star\) denote the accepted SCM-A2G server movement. By the
\(L\)-smoothness of \(F\), we have
\begin{equation}
F(\boldsymbol{\theta}_{t+1})
\leq
F(\boldsymbol{\theta}_t)
+
\left\langle
\nabla F(\boldsymbol{\theta}_t),
\mathbf{u}_t^\star
\right\rangle
+
\frac{L}{2}
\left\|
\mathbf{u}_t^\star
\right\|^2 .
\label{eq:appendix_smoothness_step}
\end{equation}

% ------------------------------------------------------------
% Step 2: Use the SCM fixed-point residual decomposition.
% The accepted movement u_t^* is close to a beta-scaled
% geometry-aware direction psi_t(u_t^*), up to the SCM residual.
% ------------------------------------------------------------

The SCM fixed-point relation can be written as
\begin{equation}
\mathbf{u}_t^\star
=
\beta_t
\boldsymbol{\psi}_t(\mathbf{u}_t^\star)
+
\mathbf{e}_{\mathrm{SCM},t},
\label{eq:appendix_scm_residual_decomposition}
\end{equation}
where \(\mathbf{e}_{\mathrm{SCM},t}\) denotes the fixed-point residual error.
The direction term admits the decomposition
\begin{equation}
\boldsymbol{\psi}_t(\mathbf{u}_t^\star)
=
-\nabla F(\boldsymbol{\theta}_t)
+
\boldsymbol{\Xi}_t .
\label{eq:appendix_direction_decomposition}
\end{equation}

Substituting \eqref{eq:appendix_direction_decomposition} into
\eqref{eq:appendix_scm_residual_decomposition} gives
\begin{equation}
\mathbf{u}_t^\star
=
-\beta_t\nabla F(\boldsymbol{\theta}_t)
+
\beta_t\boldsymbol{\Xi}_t
+
\mathbf{e}_{\mathrm{SCM},t}.
\label{eq:appendix_update_expansion}
\end{equation}

% ------------------------------------------------------------
% Step 3: Expand the descent inner product.
% The first term gives descent, and the remaining terms are errors.
% ------------------------------------------------------------

Therefore,
\begin{equation}
\begin{aligned}
\left\langle
\nabla F(\boldsymbol{\theta}_t),
\mathbf{u}_t^\star
\right\rangle
=
&-\beta_t
\left\|
\nabla F(\boldsymbol{\theta}_t)
\right\|^2 \\
&+
\beta_t
\left\langle
\nabla F(\boldsymbol{\theta}_t),
\boldsymbol{\Xi}_t
\right\rangle \\
&+
\left\langle
\nabla F(\boldsymbol{\theta}_t),
\mathbf{e}_{\mathrm{SCM},t}
\right\rangle .
\end{aligned}
\label{eq:appendix_inner_product_expansion}
\end{equation}

% ------------------------------------------------------------
% Step 4: Bound the perturbation term.
% Young's inequality separates the gradient norm from the error norm.
% ------------------------------------------------------------

Using Young's inequality and Assumptions~1--5, the total perturbation term
satisfies
\begin{equation}
\begin{aligned}
\mathbb{E}\!\left[
\beta_t
\left\langle
\nabla F(\boldsymbol{\theta}_t),
\boldsymbol{\Xi}_t
\right\rangle
\right]
\leq\;&
\frac{\beta_t}{4}
\mathbb{E}\!\left[
\left\|
\nabla F(\boldsymbol{\theta}_t)
\right\|^2
\right] \\
&+
C_1\beta_t
\left(
\sigma_l^2+\sigma_g^2+\rho_q^2G^2
\right) \\
&+
C_2\beta_t L^2
\mathbb{E}\!\left[
\left\|
\mathbf{u}_t^\star
\right\|^2
\right].
\end{aligned}
\label{eq:appendix_total_error_bound_option_a}
\end{equation}

% ------------------------------------------------------------
% Step 5: Bound the SCM residual contribution using Option A.
% Important:
% R_{SCM,t} is treated as a descent-scale residual contribution.
% Therefore, the residual term appears as C_3 R_{SCM,t}^2,
% not as (C_3/beta_t) R_{SCM,t}^2.
% ------------------------------------------------------------

By the SCM residual-control condition in Assumption~5, the fixed-point residual
contribution satisfies
\begin{equation}
\mathbb{E}\!\left[
\left\langle
\nabla F(\boldsymbol{\theta}_t),
\mathbf{e}_{\mathrm{SCM},t}
\right\rangle
\right]
\leq
\frac{\beta_t}{4}
\mathbb{E}\!\left[
\left\|
\nabla F(\boldsymbol{\theta}_t)
\right\|^2
\right]
+
C_3R_{\mathrm{SCM},t}^2 .
\label{eq:appendix_residual_bound_option_a}
\end{equation}

% ------------------------------------------------------------
% Step 6: Combine the bounds.
% The negative descent term is -beta_t ||grad F||^2.
% The two Young-type bounds each use beta_t/4 of the gradient norm.
% Hence the remaining descent coefficient is -beta_t/2.
% ------------------------------------------------------------

Substituting
\eqref{eq:appendix_inner_product_expansion},
\eqref{eq:appendix_total_error_bound_option_a}, and
\eqref{eq:appendix_residual_bound_option_a}
into \eqref{eq:appendix_smoothness_step}, and absorbing universal constants
into \(C_1,C_2,C_3\), yields
\begin{equation}
\begin{aligned}
\mathbb{E}\!\left[
F(\boldsymbol{\theta}_{t+1})
\right]
\leq\;&
\mathbb{E}\!\left[
F(\boldsymbol{\theta}_{t})
\right]
-
\frac{\beta_t}{2}
\mathbb{E}\!\left[
\left\|
\nabla F(\boldsymbol{\theta}_t)
\right\|^2
\right] \\
&+
C_1\beta_t
\left(
\sigma_l^2+\sigma_g^2+\rho_q^2G^2
\right) \\
&+
C_2\beta_t L^2
\mathbb{E}\!\left[
\left\|
\mathbf{u}_t^\star
\right\|^2
\right]
+
C_3R_{\mathrm{SCM},t}^2 .
\end{aligned}
\label{eq:appendix_scm_descent_final_option_a}
\end{equation}

This proves Lemma~\ref{lemma:scm_one_step_descent}.

% ============================================================
% Theorem proof
% ============================================================

\subsection{Proof of Theorem~\ref{theorem:main_convergence_result}}

% ------------------------------------------------------------
% Step 1: Set beta_t = beta and sum the one-step descent result.
% ------------------------------------------------------------

For a constant geometry gain \(\beta_t=\beta\), summing
\eqref{eq:appendix_scm_descent_final_option_a} over
\(t=0,\ldots,T-1\) gives
\begin{equation}
\begin{aligned}
\mathbb{E}
\left[
F(\boldsymbol{\theta}_{T})
\right]
\leq\;&
F(\boldsymbol{\theta}_{0})
-
\frac{\beta}{2}
\sum_{t=0}^{T-1}
\mathbb{E}
\left[
\left\|
\nabla F(\boldsymbol{\theta}_t)
\right\|^2
\right] \\
&+
C_1\beta T
\left(
\sigma_l^2+\sigma_g^2+\rho_q^2G^2
\right) \\
&+
C_2\beta L^2
\sum_{t=0}^{T-1}
\mathbb{E}
\left[
\left\|
\mathbf{u}_t^\star
\right\|^2
\right]
+
C_3
\sum_{t=0}^{T-1}
R_{\mathrm{SCM},t}^2 .
\end{aligned}
\label{eq:appendix_summed_descent_option_a}
\end{equation}

% ------------------------------------------------------------
% Step 2: Use lower boundedness of the objective.
% Since F(theta_T) >= F^*, replace the left-hand side by F^*.
% ------------------------------------------------------------

Since \(F\) is lower bounded by \(F^\star\), we have
\[
\mathbb{E}
\left[
F(\boldsymbol{\theta}_{T})
\right]
\geq
F^\star .
\]
Rearranging \eqref{eq:appendix_summed_descent_option_a} and dividing by
\(\beta T/2\) yields
\begin{equation}
\begin{aligned}
\frac{1}{T}
\sum_{t=0}^{T-1}
\mathbb{E}
\left[
\left\|
\nabla F(\boldsymbol{\theta}_t)
\right\|^2
\right]
\leq\;&
\frac{2\left(F(\boldsymbol{\theta}_0)-F^\star\right)}
{\beta T} \\
&+
2C_1
\left(
\sigma_l^2+\sigma_g^2+\rho_q^2G^2
\right) \\
&+
\frac{2C_2L^2}{T}
\sum_{t=0}^{T-1}
\mathbb{E}
\left[
\left\|
\mathbf{u}_t^\star
\right\|^2
\right] \\
&+
\frac{2C_3}{\beta T}
\sum_{t=0}^{T-1}
R_{\mathrm{SCM},t}^2 .
\end{aligned}
\label{eq:appendix_scm_convergence_bound_option_a}
\end{equation}

% ------------------------------------------------------------
% Step 3: Absorb numerical constants.
% This gives the same form as the theorem statement in the main paper.
% ------------------------------------------------------------

Absorbing fixed numerical factors into \(C_1,C_2,C_3\), we obtain
Eq.~\eqref{eq:main_convergence_bound}.

% ------------------------------------------------------------
% Step 4: Simplified residual rate.
% Under the stopping condition R_SCM,t <= epsilon_SCM,
% the average residual term is bounded by epsilon_SCM^2.
% Therefore, the final residual contribution scales as
% O(epsilon_SCM^2 / beta), not O(epsilon_SCM^2 / beta^2).
% ------------------------------------------------------------

If the SCM solver is run until
\(R_{\mathrm{SCM},t}\leq \varepsilon_{\mathrm{SCM}}\) for all \(t\), then
\[
\frac{1}{T}
\sum_{t=0}^{T-1}
R_{\mathrm{SCM},t}^2
\leq
\varepsilon_{\mathrm{SCM}}^2.
\]
Therefore, the residual contribution satisfies
\[
\frac{C_3}{\beta T}
\sum_{t=0}^{T-1}
R_{\mathrm{SCM},t}^2
\leq
\frac{C_3}{\beta}
\varepsilon_{\mathrm{SCM}}^2,
\]
which gives
\[
\mathcal{O}
\left(
\frac{\varepsilon_{\mathrm{SCM}}^2}{\beta}
\right).
\]
This yields the simplified convergence rate in
Eq.~\eqref{eq:scm_simplified_rate}.
\vspace{12pt}

\end{document}